\documentclass[sigconf]{acmart}
\AtBeginDocument{%
  }

\usepackage{multirow}
\usepackage{threeparttable}
\usepackage{subfig}
\usepackage{enumitem}

\usepackage[textsize=tiny]{todonotes}
\usepackage{mathtools}

\usepackage{amsthm}
\usepackage{color}

\newcommand{\cut}[1]{{}}

\newcommand{\method}{\text{H$^3$GNNs }}

\makeatletter
\let\@@span\span
\def\sp@n{\@@span\omit\advance\@multicnt\m@ne}
\makeatother

\newcommand{\bc}{\begin{center}}
\newcommand{\ec}{\end{center}}

\newcommand{\bdm}{\begin{displaymath}}
\newcommand{\edm}{\end{displaymath}}

\newcommand{\beq}{\begin{equation}}
\newcommand{\eeq}{\end{equation}}

\newcommand{\bfl}{\begin{flushleft}}
\newcommand{\efl}{\end{flushleft}}

\newcommand{\bt}{\begin{tabbing}}
\newcommand{\et}{\end{tabbing}}

\newcommand{\beqn}{\begin{align}}
\newcommand{\eeqn}{\end{align}}

\newcommand{\beqs}{\begin{align*}} 
\newcommand{\eeqs}{\end{align*}}  

\begin{document}


\title{H$^3$GNNs: Harmonizing Heterophily and Homophily  in GNNs via  Joint Structural Node Encoding and Self-Supervised Learning}


\author{Rui Xue}
\affiliation{%
  \institution{North Carolina State University}
  \city{Raleigh}
  \country{US}}

\author{Tianfu Wu}
\affiliation{%
  \institution{North Carolina State University}
  \city{Raleigh}
  \country{US}}


\begin{abstract}

Graph Neural Networks (GNNs) struggle to balance heterophily and homophily in representation learning, a challenge further amplified in self-supervised settings. We propose \textbf{\method}, an end-to-end self-supervised learning framework that harmonizes both structural properties through two key innovations: \textbf{(i) Joint Structural Node Encoding.}  We embed nodes into a unified space combining linear and non-linear feature projections with K-hop structural representations via a Weighted Graph Convolution Network (WGCN). A cross-attention mechanism enhances awareness and adaptability to heterophily and homophily.
\textbf{(ii) Self-Supervised Learning Using Teacher-Student Predictive Architectures with Node-Difficulty Driven Dynamic Masking Strategies.} We use a teacher-student model, the student sees the masked input graph and predicts node features inferred by the teacher that sees the full input graph in the joint encoding space. To enhance learning difficulty, we introduce two novel node-predictive-difficulty-based masking strategies.
Experiments on seven benchmarks (four heterophily datasets and three homophily datasets) confirm the effectiveness and efficiency of \method across diverse graph types. Our \method achieves overall state-of-the-art performance on the four heterophily datasets, while retaining on-par performance to previous state-of-the-art methods on the three homophily datasets.

\end{abstract}

\settopmatter{printacmref=false} 
\renewcommand\footnotetextcopyrightpermission[1]{} 
\pagestyle{plain} 

\maketitle

\section{Introduction}
Graph neural networks (GNNs) have recently garnered attention as powerful frameworks for representation learning on graph-structured data~\cite{hamilton2020graph,gasteiger2018predict,velivckovic2017graph}. By effectively capturing and leveraging relational information, they demonstrate considerable advantages for a variety of core graph learning tasks, such as node classification, link prediction, and graph classification~\cite{kipf2016semi,gasteiger2018combining,velivckovic2017graph,wu2019simplifying}. Moreover, GNNs have achieved noteworthy success in a range of practical domains, including recommendation systems, molecular biology, and transportation~\cite{tang2020knowing, sankar2021graph,fout2017protein,wu2022graph, zhang2024linear}.

\begin{figure*}[t]
    \begin{minipage}[c]{0.7\textwidth}
    \centering
    \includegraphics[width=1.0\linewidth]{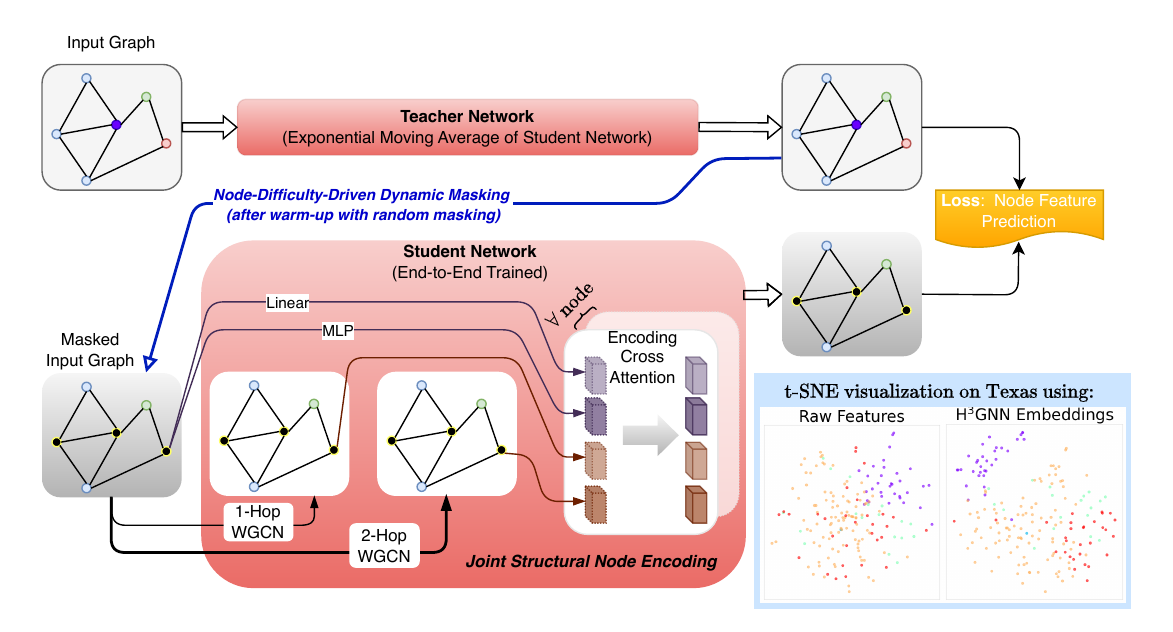}
    \end{minipage}\hfill
    \begin{minipage}[c]{0.27\textwidth}
    \caption{Illustration of our proposed \method, a simple yet effective framework for SSL from graphs of mixed structural properties (heterophily and homophily).  
    In the right-bottom, we show the t-SNE visualization on the Texas dataset using the raw node features and the features learned by our \method, which shows the effectiveness of our proposed method.
    Best viewed in color.
    See text for details.  
    }\label{fig:workflow}
    \end{minipage}
\end{figure*}

Traditional GNNs primarily employ a semi-supervised approach, achieving strong performance across various benchmarks and tasks \cite{xu2018powerful, li2021training, sun2021scalable, xue2023lazygnn}. While most of them achieve strong performance on \textbf{homophilic} graphs, they often experience a significant performance drop when applied to \textbf{heterophilic} graphs. This is because their message-passing mechanism heavily relies on the \textit{homophily assumption}—that connected nodes share similar features. However, many graphs exhibit heterophilic properties, which follow an entirely opposite pattern (i.e., love of the different). This severely limits the generalization ability of these models. As a result, several approaches, such as WRGAT \cite{suresh2021adversarial} and H2GCN \cite{zhu2020beyond}, have been proposed to address this issue. However, they are unable to achieve a win-win outcome on both heterophilic and homophilic graphs. More importantly, most real-world graphs exhibit a mix of both homophilic and heterophilic patterns, which can exist at either the local or global scale. This further increases the difficulty of learning on such graphs, entailing new methods of harmonizing both structural properties of graphs. 

Besides challenges related to structural patterns, traditional GNNs heavily rely on the availability of labeled data, making them vulnerable to significant performance degradation when labeled data is scarce \cite{xue2023efficient}. However, collecting and annotating manual labels in large-scale datasets (e.g., citation and social networks) is prohibitively expensive, or requires substantial domain expertise (e.g., chemistry and medicine). Additionally, these models are vulnerable to label-related noise, further undermining the robustness of graph semi-supervised learning.
Self-supervised learning (SSL) has achieved widespread adoption in the fields of  natural language processing (NLP)~\cite{devlin2019bert} and computer vision (CV)~\cite{he2022masked}. Unlike traditional supervised learning, which relies heavily on large amounts of labeled data, SSL leverages unlabeled data by creating proxy/pretext tasks that exploit intrinsic structures of raw data themselves as labels (such as the next word/token prediction, and masked word/image modeling). This approach not only addresses the dependency on the quantity of labeled data but also efficiently utilizes the inherent patterns and relationships within the data, enabling the development of richer representations without need for explicit annotations. Furthermore, they can also encourage the model to learn more robust representations, thereby reducing its sensitivity to noise and/or labeling bias. Building on these advantages, they have shown remarkable promise in various graph representation learning applications.

Recently, several Graph SSL approaches \cite{velickovic2019deep, zhu2020deep, hou2022graphmae, chen2022towards, xiao2024simple, tang2022graph, xiao2022decoupled, yuan2023muse} have been introduced to address the mixed structural and label challenges discussed above, demonstrating robust effectiveness on both homophilic and heterophilic graphs. 
{In these SSL frameworks for graph data, proxy tasks play a crucial role, as they provide supervision in the absence of explicit labels. Typical proxy tasks include: (1) Contrastive Losses in Contrastive learning : The model is encouraged to bring representations of similar nodes (or augmented views) closer while pushing apart those of dissimilar nodes; (2) Feature/edge Reconstruction in generative learning: Given a masked input, the model is trained to reconstruct the original node features /predict the existence or weight of edges.}
However, both approaches become problematic under certain circumstances. Contrastive learning's success hinges on relatively complex training strategies, including the careful design of negative samples and a strong reliance on high-quality data augmentation \cite{grill2020bootstrap}. However, these requirements are often challenging to meet in graph settings \cite{hou2022graphmae}, which limits the broader application of contrastive learning in this domain. They can also suffer from representation collapse, where the network converges to a state where all outputs become similar, rendering the learned features uninformative. On the other hand, generative learning methods aim to reconstruct graph data but often face challenges due to reconstruction space mismatch. This arises because the decoder demands intricate design choices and frequently struggles to fully recover the original feature space \cite{hou2022graphmae, hou2023graphmae2}. Moreover, these methods are also prone to well-known issues such as training instability, overfitting, and mode collapse.
\textit{A more detailed overview of related works is provided in Appendix~\ref{app:related}.}

To address these challenges, we propose \method which incorporates five key design components (Fig.~\ref{fig:workflow}):
\begin{itemize} [leftmargin=*]
	\item	\textbf{Masked Node Modeling} – Inspired by masked word/image modeling in NLP and computer vision (CV), we mask node features in a graph while preserving topology. Initially, random masking is applied, transitioning to a \textit{node-difficulty-driven dynamic masking} strategy to enhance learning. The goal is to reconstruct masked features in a latent space that is heterophily- and homophily-aware, rather than relying on raw, often noisy input features.
	\item	\textbf{Teacher-Student Predictive Architecture} – A teacher network (exponential moving average of the student) observes the full graph, while the student network sees only the masked graph. The student learns by predicting the teacher’s outputs, ensuring stable representation learning and alignment of learning spaces.
	\item	\textbf{Node Feature Prediction in the Latent Space} – Instead of predicting masked nodes alone (as adopted in NLP and CV), we compute prediction loss for the entire graph, addressing higher ambiguities in graphs compared to text and images.
	\item	\textbf{Joint Structural Node Encoding} – To enhance representation learning, we combine linear and MLP-based node feature transformations with K-hop structural projections via weighted GCNs. A Transformer block integrates these representations via cross-attention, ensuring adaptability to homophily and heterophily while maintaining efficiency.
	\item	\textbf{Node-Difficulty Driven Dynamic Masking} – Unlike vision/NLP approaches that rely on random masking, we propose two novel adaptive masking strategies based on node prediction difficulty. This ensures challenging yet meaningful training tasks, leading to robust and generalizable representations.
\end{itemize}

\noindent\textbf{Our Contributions.}
This paper makes four main contributions to the field of SSL from graphs with diverse heterophilic and homophilic properties:  
(i) \textbf{Teacher-Student Whole-Graph Predictive Architecture} – A self-supervised teacher-student framework stabilizes learning by having the student predict the teacher’s outputs, mitigating representation collapse.
Unlike prior masked modeling approaches, we introduce a graph-wise prediction loss in the latent space, improving feature recovery and handling structural ambiguities.
(ii) \textbf{Joint Structural Node Encoding} – A hybrid encoder combining MLP-based transformations, K-hop structural embeddings, and Transformer-based cross-attention ensures adaptability across different graph structures.
(iii) \textbf{Adaptive Masked Node Modeling} – We introduce two node-difficulty-driven dynamic masking strategies, improving over random masking to enhance representation learning on both homophilic and heterophilic graphs.
(iv) \textbf{State-of-the-Art Performance} – Extensive experiments on seven benchmark datasets confirm strong performance, demonstrating improved training effectiveness, efficiency and generalization.

\section{Method}
\label{sec:method}
In this section, we first define the problem of graph SSL (Sec.~\ref{sec:definition}). Then, we present details of our proposed \method (Sec.~\ref{sec:method}). We also provide theoretical analyses for our method in comparison with alternative encoder-decoder models (Sec.~\ref{sec:theorem}). 

\subsection{Problem Definition of Graph SSL}\label{sec:definition}
We focus on node classification and clustering tasks, and aim to learn better node features under the SSL setting in terms of linear probing and $k$-mean clustering. By linear probing, it means that after SSL, the SSL model will be frozen and we only train a linear classifier for node classification. The $k$-mean clustering evaluation is a zero-shot protocol, which directly uses the learned SSL embeddings for nodes to do clustering with the number of clusters equal to the number of total classes and evaluates the clustering results. 

Denote by $G=(V, E)$, a graph with the node set $V$ consisting of $N$ nodes (i.e. the cardinality $|V|=N$),  and the edge set $E$ defining the connectivity between nodes. 
Denote by  $f(v)\in\mathbb{R}^d (\forall v\in V)$ the raw $d$-dim input node feature vector. 
Denote by $\mathcal{Y}$ the label set for nodes, and we have labels for a subset of nodes $\mathbf{V}\subseteq V$. The label of a node is denoted as $\ell(v)\in\mathcal{Y}$ ($\forall v\in \mathbf{V}$).  In graph SSL, we do not use any labels. The labels are only used in linear probing to train the linear classifier and evaluate its performance, as well as the performance of $k$-mean clustering results. 

Let $F(\cdot;\Theta)$ be a model to be trained in graph SSL with learnable parameters $\Theta$. For each node $v\in V$ of a graph $G$, let $F(v)\in \mathbb{R}^D$ denote the learned $D$-dim node features by the SSL model. The objective  of graph SSL is to ensure the learned node features $F(v)$ are significantly better than the raw input features $f(v)$ for downstream tasks in term of both training a linear classifier and direct $k$-mean clustering. The challenge of graph SSL is how to parameterize $F(\cdot;\Theta)$ (i.e., architectural designs), and how to actually train $F(\cdot;\Theta)$ using unlabeled graphs and to achieve good performance on diverse graphs (i.e., proxy task designs). We address these challenges by proposing a simple yet effective method in this paper.

\subsection{Our Proposed \method}\label{sec:method}

As illustrated in Fig.~\ref{fig:workflow}, we adopt the masked node modeling strategy for graph SSL using a teacher-student network configuration. 

\subsubsection{Masked Node Modeling.\label{sec:mnm}} For an input graph $G=(V, E)$, we us a node-wise mask $\mathcal{M}$ (the masking strategy will be elaborated in Sec.~\ref{sec:masking}), and the node set $V$ is divided into masked nodes (i.e., $\mathcal{M}(v)=1$) and unmasked ones (i.e., $\mathcal{M}(v)=0$), $V=\mathbb{V}\cup V^c$ and $\mathbb{V}\cap V^c=\emptyset$,  with the edge set $E$ kept unchanged. For the masked nodes in $\mathbb{V}$, we replace their raw input features by learnable parameters with random initialization, e.g., from the white noise distribution, $f(v)\sim \mathcal{N}(0,1), \forall v\in \mathbb{V}$. Let $\mathbb{G}=(\mathbb{V}\cup V^c, E)$ denote the masked input graph. The masked graph is generated at each iteration in training. 

\subsubsection{Teacher-Student Predictive Architecture.} To facilitate learning a proper latent space, we leverage a teacher-student predictive architecture. Denote the student network and the teacher network by $S(\cdot;\Phi)$ and $T(\cdot;\Psi)$ respectively. The student network sees the masked input graph $\mathbb{G}$, while the teacher network sees the full graph $G$. The teacher network has the exactly same network configuration as the student, and is not trained, but uses the exponential moving average (EMA) of the student network to ensure the stability of training and the convergence of the same latent space, as commonly done in SSL from images~\cite{he2020momentum,assran2023self,bardes2024revisiting}. The student network is trained end-to-end in an epoch-wise manner with a predefined number of total epochs $I$ (e.g., $I=200$) with the model parameters iteratively updated, $\Phi_1, \cdots, \Phi_I$. At the very beginning ($i=0$), the student and teacher networks start from the same random initialization, i.e., $\Psi_0=\Phi_0$. We have, 
\begin{align}
    \Psi_{i} & = \alpha \cdot \Psi_{i-1} + (1-\alpha)\cdot \Phi_i, \quad i=1,2,\cdots, I,  \label{eq:ema}
\end{align}
where $\alpha$ is the momentum hyperparameter which is chosen to ensure that after the total $I$ epochs, the teacher network and the student network converge to almost the same point in the parameter space, i.e., $\Psi_I\approx \Phi_I$.

\subsubsection{Node Feature Prediction in the Latent Space.} Since there are no labels available in graph SSL. To estimate the student network's parameters $\Phi$, a proxy or pretext task is entailed. 
With the outputs from both the teacher network and the student network, one common approach for node feature prediction is to consider masked nodes in $\mathbb{V}$ only, as done by the word prediction in BERT~\cite{devlin2019bert} and BEiT~\cite{bao2021beit}. However, graph nodes are inherently more ambiguous because their interconnections create strong dependencies, leading to interactions between masked and unmasked nodes during message passing. Predicting only masked nodes' features between the student network and the teacher network is thus suboptimal for learning a more meaningful latent space. We propose to compute the node feature prediction loss in the latent space based on the entire graph. For a node $v\in V=\mathbb{V}\cup V^c$, for notation simplicity, denote the outputs from the student and teacher network by $S(v)\in \mathbb{R}^D$ and $T(v)\in \mathbb{R}^D$ respectively. Our proxy task loss for training the student network $\Phi$ is defined by,
\begin{equation}
    \mathcal{L}(\Phi) = \frac{1}{N}\sum_{v\in V} ||S(v)-T(v)||_2^2, \label{eq:loss}
\end{equation}
where the teacher network serves as the learnable loss functions in graph SSL to substitute the role of supervised labels. 

\subsubsection{Joint Structural Node Encoding.} With the above architectural designs and loss function choices, we need to ensure the representation learning capability  of the student and teacher networks via SSL from diverse graphs, while maintaining scalability and efficiency for large graphs. Real-world graph datasets do not strictly adhere to either homophily (similar nodes connected) or heterophily (dissimilar nodes connected). We focus on the expressivity of node features in graph SSL in terms of inducing heterophily and homophily awareness and adaptivity in $S(v)$ against the raw input features $f(v)$ for downstream tasks.  

In graph theory, homophily means that adjacent nodes $(u, v)$ tend to have similar features, and heterophily means the opposite, which can be reflected in the graph normalized Laplacian quadratic form, 
\begin{equation}
    f^\top\cdot L_{sym}\cdot f =\sum_{(u,v)\in E} A_{uv}\Bigl(\frac{f(u)}{\sqrt{d_u}}-\frac{f(v)}{\sqrt{d_v}}\Bigr)^2, 
\end{equation}
where $L_{sym}$ represents the symmetric normalized Laplacian, $L_{sym}=\mathbb{I} - D^{-\frac{1}{2}}\cdot A\cdot D^{-\frac{1}{2}}$ with the degree matrix $D$ and the adjacency matrix $A$, and $\mathbb{I}$  an identity matrix. $d_u$ and $d_v$ are the node degrees.

If a graph is homophilic, then  $f(u) \approx f(v)$  for adjacent nodes and $f^\top\cdot L_{sym}\cdot f$ is small, leading to low-frequency dominance. 
If a graph is heterophilic, the differences  $\Bigl( \frac{f(u)}{\sqrt{d_u}} - \frac{f(v)}{\sqrt{d_v}} \Bigr)$  are larger, leading to high-frequency dominance. 

Our proposed joint structural node encoding consists of,
\begin{itemize}[leftmargin=*]
    \item \textit{Heterophily-Preserved Homophily Awareness via Weighted GCN Feature Transformation}. The traditional GCN \cite{kipf2016semi} has been proven to act as a simple and efficient low-pass filter, making it good for homophilic graphs but less effective for heterophilic graphs. However, due to the complex patterns existing in  real-world graph datasets, both homophilic and heterophilic properties may simultaneously exist for the same node. For example, a node $u$ may have a neighbor $v$ that shares similar node features, while also having another neighbor $w$ whose features are totally different from the target node $u$. This makes uniform smoothing via GCN suboptimal.
    To address this, we propose to use Weighted-GCN (WGCN) by introducing learnable edge weights to adaptively control message passing, which allows the model to balance smoothing (low-pass filtering) and preserving high-frequency details, making it more effective in diverse graph structures. We have, 
    \begin{align}
       \text{GCN: }\quad H^{(l+1)} & = \sigma(\Tilde{A}\cdot H^{(l)}\cdot W^{(l)}), \\
       \text{WGCN: }\quad H^{(l+1)} & = \sigma(\Tilde{\mathbb{A}}\cdot H^{(l)}\cdot W^{(l)}),
    \end{align}
    where $H^{(l)}$ is the node feature matrix at layer $l$ (e.g., $H^{(0)}\in \mathbb{R}^{N\times d}$ with $H^{(0)}(v)=f(v), \forall v\in V$). $W^{(l)}$ is the trainable weight matrix. $\sigma(\cdot)$ is an activation function such as ReLU.  $\tilde{A} = \tilde{D}^{-1/2}(A + \mathbb{I})\tilde{D}^{-1/2} $ is the normalized adjacency matrix, incorporating self-loops, with $\tilde{D}$ being the degree matrix for $A + \mathbb{I}$. 
    
    $\Tilde{\mathbb{A}}_{ij}=e_{ij}, \text{ if } \Tilde{A}_{ij}\neq 0$ is a learnable parameter that adjusts the edge weight dynamically, meaning the model learns how much influence each neighbor should have, instead of treating all edges equally. Note that $\Tilde{\mathbb{A}}_{ij}=0, \text{ if } \Tilde{A}_{ij}= 0$, meaning the topology of the graph is retained. 

    In a homophilic setting, the model can keep weights high for similar neighbors. In a heterophilic setting, it can reduce the weight for dissimilar neighbors, preventing oversmoothing. This flexibility allows the model to handle complex structural patterns more effectively.

    We use both a 1-layer WGCN (for 1-hop message passing) and a 2-layer WGCN (for 2-hop message passing), as illustrated in Fig.~\ref{fig:workflow}, to balance model complexity and expressivity. The learnable weights in WGCN are specified to ensure $H^{(1)}, H^{(2)} \in \mathbb{R}^{N\times C}$, where the output dimension $C$ is chosen to control model complexity (e.g. $C=128$).

    \item \textit{Heterophily-Targeted Awareness via Node-Wise Feature Transformation}. For a node $v\in V$ with input features $f(v)\in \mathbb{R}^d$, we apply both a linear projection and a non-linear projection, 
    \begin{align}
        f^{(Linear)}(v) &=\text{Linear}\bigl(f(v)\bigr)= W_1\cdot f(v) + b_1, \\
        f^{(Mlp)}(v) &= \text{MLP}\bigl(f(v)\bigr) \label{eq:mlp} \\
        \nonumber & = W_3\cdot \text{Act}\Bigl(W_2\cdot f(v) + b_2\Bigr) + b_3,
    \end{align}
    where $W_1\in \mathbb{R}^{C\times d}, b_1\in \mathbb{R}^C$, and we use a MLP for the nonlinear projection with $W_2\in \mathbb{R}^{4\cdot C\times d}, b_2\in \mathbb{R}^{4\cdot C}$, and $W_3\in \mathbb{R}^{C\times 4\cdot C}, b_3\in \mathbb{R}^C$. $\text{Act}(\cdot)$ is an non-linear activation function such as GELU~\cite{HendrycksG16}.
        
    $f^{(Mlp)}(v)$ plays the role of high-pass filters to the graph signal. For high-pass filter design, some existing works \cite{liu2022beyond} employ inverse filters of GCN, however, we find that a simple MLP (Eqn.~\ref{eq:mlp}) demonstrates strong performance on heterophilic graphs (see Table \ref{tab:major})  and has also proven to be effective \cite{chen2022towards}. We further include $f^{(Linear)}(v)$ accounting for the ``DC'' component preservence. 
    \item \textit{Heterophily and Homophily Adaptivity via Transformer.} With the feature transformations above, for a node $v\in V$, we map it into a joint space,  
    \begin{equation}
        \mathbf{f}(v) = \left[f^{(Linear)}(v) ||  f^{(Mlp)}(v) || H^{(1)}(v) || H^{(2)}(v) \right], \label{eq:encoding}
    \end{equation}
    where $\cdot || \cdot$ concatenates the four feature transformations along the feature channels, and we have $ \mathbf{f}(v)\in \mathbb{R}^D$ with $D=4\cdot C$. To mix and re-calibrate the different types of features per node to induce heterophily and homophily awareness and adaptivity, we treat the each projection output as a ``token" (i.e., $4$ tokens as illustrated in Fig.~\ref{fig:workflow}), and apply a Transformer block~\cite{vaswani2017attention} to realize the encoding cross-attention. By doing so, we maintain the efficiency with our novel filter-wise attention mechanism, which is different from many graph transformer works that aim to capture node-to-node attention and suffer from scalability caused by the quadratic complexity of the Transformer model.

    In general, we can view $\mathbf{f}(v)$ as a $S$-token sequence in a $C$-dim embedding space ($S=4$ in our case), re-denoted by $X_{S,C}$, the Transformer block is defined by, 
    \begin{align}
        Z_{h,S,c} &= \text{Rearrange}\bigl(\text{Linear}(X;\theta_Z)\bigr)_{S,C\rightarrow h,S,c}, \\
        \nonumber & \qquad \hfill  Z\in \{Q, K, V\}\\
        \text{Attn}(X) & = \text{Softmax}(\frac{Q\cdot K^T}{\sqrt{c}}), \label{eq:attn}\\ 
        \text{MHSA}(X) & = \text{Rearrange}\Bigl(\text{Attn}(X)\cdot V\Bigr)_{h,S,c\rightarrow S,C}\, , \\
        Y_{S,C} &= X_{S,C} + \text{Norm}\Biggl(\text{Linear}\Bigl( \text{MHSA}\bigl( X_{S,C}\bigr);\theta_{proj}\Bigr)\Biggr)\, , \label{eq:mhsa}\\
        X_{S,C} & = Y_{S,C} + \text{Norm}\Bigl(\text{MLP}\bigl( Y_{S,C}\bigr)\Bigr), \label{eq:transformer}
        \end{align}
        where $C=h\times c$ with the head dimension being $c$ for each of the $h$ heads. $\text{Rearrange}(\cdot)$ reshapes the input tensor's shapes for forming multi-heads in computing the query, key and value, or for fusing the multi-head after attention, as the subscripts shown. $\text{MLP}(\cdot)$ is defined as Eqn.~\ref{eq:mlp}. $\text{Norm}(\cdot)$ is a feature normalization funciton such as the layer norm~\cite{ba2016layer}. 
\end{itemize}

Based on Eqn.~\ref{eq:transformer}, the outputs of the student and teacher network  respectively are,
\begin{align}
    S(v) &= \text{Flatten}(X^{\mathbb{G}}_{S,C}), \quad v\in \mathbb{G},\\
    T(v) &= \text{Flatten}(X^{G}_{S,C}),\quad v\in G,
\end{align}
which are used in computing the loss (Eqn.~\ref{eq:loss}) in optimization.

\subsubsection{Node-Difficulty Driven Dynamic Node Masking.\label{sec:masking}} Masking strategies are critical for the success of SSL. In computer vision and NLP, random masking with sufficient high masking ratios~\cite{devlin2019bert,he2022masked} is often used, which leads to hard proxy tasks to be solved via learning meaningful representations. For diverse graphs, due to intrinsic ambiguity and unknown mixed structural properties, random masking alone will not be sufficient to guide graph SSL. We propose two novel dynamic masking strategies to compute the mask $\mathcal{M}_i$ (Sec.~\ref{sec:mnm}) at each iteration for generating the masked input graph $\mathbb{G}_i$ from the full input graph $G$. These strategies adaptively consider each node's learning difficulty based on the joint encoding attention or the prediction loss, ensuring that the prediction task is challenging enough to learn robust representations with excellent generalization capabilities. 

Let $R$ be the overall node masking ratio ($R\in (0, 1)$), which is a hyperparameter for different graphs (e.g., $R=0.5$). We will mask $M=\lfloor N\times R\rfloor$ nodes. We warm up the training with purely random masking for a predefined number of epochs. Afterwards, we adopt the exploitation-exploration strategy, where we exploit two node-difficulty driven dynamic masking approaches, combined with the purely random masking (i.e., the exploration approach). Let $r$ be the exploitation ratio ($r\in [0, 1]$),  we will first select $m=M\times r$ nodes using the exploitation approach, and the remaining $M-m$ nodes are randomly sampled from the set of available $N-m$ nodes (without replacement). 
\begin{itemize}[leftmargin=*]
    \item \textbf{Node Feature Prediction Loss Driven Masking.} Based on Eqn.~\ref{eq:loss}, we define the difficulty score of a node $v$ after the current iteration by,
    \begin{equation}
        \text{Diffi}(v) = ||S(v)-T(v)||_2^2,
    \end{equation}
    which is used to compute the mask for the next iteration. We sort the nodes  $v\in V$ based on $\text{Diffi}(v)$ in a decreasing order, and the select the first $m$ nodes to mask. 
    
    This approach ensures that the model focuses on nodes where the student network's understanding is significantly lacking compared to the teacher network, thereby driving the student network to improve its representations where it is most deficient. Our \method will be denoted by \textbf{\method+Diffi} in comparisons. 

    However, this approach does not entirely prevent the issue of over-focusing on a small subset of high-difficulty nodes while neglecting the overall data diversity. To address this, we seek a probabilistic solution in the next approach.
    
    \item \textbf{Masking via Bernoulli Sampling with Node-Difficulty Informed Success Rate.} Let $p_v$ be the success rate of the Bernoulli distribution used for selecting the node $v\in V$ to be masked, i.e., $\mathcal{M}(v)\sim \text{Bernoulli}(p_v)$. We have, 
    \begin{align}
    p_v = p_{0} + \delta_v,  \quad p_{0} &= (1 - r) \times R, \label{eq:masking-prob}\\
    \nonumber \delta_v &= \left( \frac{\text{Diffi}(v)}{\text{Diffi}_{\max}} \right) \times r \times R,
    \end{align}
    where $p_0$ is the base success rate subject to the exploration approach, and it is the same for all nodes. $\delta_v$ is the node-difficulty based exploitation with $\text{Diffi}_{\max}$ the maximum value of the node difficulty score among all nodes.

    This approach ensures that all nodes have a base probability $p_0$ of being masked, while higher-difficulty nodes are masked with a greater chance, effectively guiding the model to focus more on learning from these challenging nodes. 

    Since this approach is a node-wise Bernoulli sampling, to prevent the worst cases in which either too few nodes or too many nodes (much greater than $M$) are actually masked, we do sanity check in the sampling process by either repeatedly sampling (if too few nodes have been masked) or early stopping. Our \method will be denoted by \textbf{\method+Prob }in comparisons. 
\end{itemize}

\subsection{Theoretical Underpinnings}\label{sec:theorem}
In this section, under common assumptions that are typically entailed in theoretical analyses of deep neural networks, we provide theoretical underpinnings of the faster convergence rates of our \method in comparison with alternative encoder-decoder based graph SSL methods such as GraphMAE \cite{hou2022graphmae, hou2023graphmae2} which aim to directly reconstruct raw input features of masked nodes.  

The encoder-decoder SSL architecture consists of an encoder network $E(\cdot; \Theta_{enc})$ and a separate decoder network $D(\cdot; \Theta_{dec})$. Let $\theta=(\Theta_{enc}, \Theta_{dec})$ collects all parameters. Given a masked graph signal $\Bar{f}$ from the input graph signal $f$ of $N$ nodes using a mask $\mathcal{M}$, its objective in general is to minimize,
\begin{equation}
\mathcal{L}_{E-D}(\theta) = \frac{1}{N}||D\bigl(E(\Bar{f};\Theta_{enc});\Theta_{dec}\bigr) - f||_2^2 \, . \label{eq:mae}
\end{equation} 
Our \method is to optimize Eqn.~\ref{eq:loss}. The convergence rates of the encoder-decoder methods and our \method can be bounded in the main theorem as follows. 

\begin{theorem}
\label{thm:convergence_comparison}
Consider the optimization of encoder-decoder based graph SSL in Eqn.~\ref{eq:mae} and our  proposed \method in Eqn.~\ref{eq:loss} under the same encoder architecture and following assumptions/conditions:
\begin{itemize}[leftmargin=*]
    \item \textbf{Gradient Smoothness and Lipschitz Continuity} for the encoder, the decoder, E.g., the encoder $E(\cdot;\Theta_{enc})$ has gradient $\beta_{E}$-smoothness (i.e., each gradient from iteration $t$ to $t+1$ changes at most linearly with respect to parameter shifts in $\Theta_{enc}$ with a coefficient $\beta_{E}$) and is $L_E$-Lipschitz continuous with respect to its input and/or parameters (i.e., differences such as $||E(\cdot;\Theta^{(t+1)}_{enc})-E(\cdot;\Theta^{(t)}_{enc})||$ can be bounded from the above as linear functions of $||\Theta^{(t+1)}_{enc}-\Theta^{(t)}_{enc}||$ with a coefficient $L_E$). 
    Similarly, we have $(\beta_D, L_D)$
    defined for the decoder.
    
    \item \textbf{Boundedness} from the above for gradients of the encoder, gradients of the decoder, and  reconstruction errors of the combined encoder-decoder. 
    
    So, $\|\nabla E(\cdot; \Theta^{(t)}_{enc})\| \leq B_{enc}$,  $\|\nabla D\bigl(E(\cdot; \Theta^{(t)}_{enc}); \Theta^{(t)}_{dec}\bigr)\| \leq B_{dec}$, and  $\|D\bigl(E(\bar{f};\Theta^{(t)}_{enc});\Theta^{(t)}_{dec}\bigr) - f\|\leq B_{Reconst}$ .

    \item \textbf{Strong Convexity} for the encoder, the decoder, and the student (and the teacher) in their parameters. 
    
    E.g., the encoder $E(\cdot;\Theta_{enc})$
    is $\mu_{E}$-strongly convex in their parameters $\Theta_{enc}$, i.e.,  $\langle \nabla E(\bar{f};\Theta_{enc}^{(t+1)}) - \nabla E(\bar{f};\Theta_{enc}^{(t)}), \Theta_{enc}^{(t+1)} - \Theta_{enc}^{(t)} \rangle \geq \mu_{E} \cdot \|\Theta_{enc}^{(t+1)} - \Theta_{enc}^{(t)}\|^2$. 
    Similarly, we have $\mu_{D}$ 
    defined for the decoder.

    \item \textbf{Approximation Error.} When only unmasked inputs are used, the composite functions, either the encoder-decoder or the teacher-student in our \method, achieve an approximation error $\epsilon_{E-D}$ (or $\epsilon_{T-S}$). 
\end{itemize}
Then, the following three results hold: 
\begin{itemize}[leftmargin=*]
    \item \textbf{A. Linear Convergence Bounds Under Strong Convexity.}
    For our \method, 
    \begin{align}
        \|\Phi^{(t+1)}-\Phi^*\|^2 &\leq (1-\frac{\mu^2_E}{\beta^2_E})\cdot \|\Phi^{(t)} - \Phi^*\|^2 
    \end{align}
    where $\alpha$ is the momentum parameter (Eqn.~\ref{eq:ema}). 
    For the encoder-decoder models, 
    \begin{align}
        \|\theta^{(t+1)} - \theta^*\|^2 &\leq \left(1 - \frac{\min(\mu^2_E, \mu^2_D)}{\max(\beta^2_E, \beta^2_D)}\right) \cdot \|\theta^{(t)} - \theta^*\|^2 
    \end{align}
    from which we can see our \method converges to the optimal solution $\Phi^*$ faster than the encoder-decoder counterpart to their optimal solutions $\Theta^*$ due to a smaller contraction factor $\left(1 - \frac{\mu^2_E}{\beta^2_E}\right)< \left(1 - \frac{\min(\mu^2_E, \mu^2_D)}{\max(\beta^2_E, \beta^2_D)}\right)$. This implies that \method can achieve a faster convergence.

    \item \textbf{B. Proxy Task Loss Bounds} under a Lipschitz-dependent assumption between the masked graph signal and the raw graph signal, $\|\bar{f}-f\| \leq \delta$. For our \method, 
    \begin{align}
        \|S(\bar{f};\Phi) - T(f;\Psi)\| \leq L_E \cdot \delta + \epsilon_{T-S}. 
    \end{align}
    For the encoder-decoder models, 
    \begin{equation}
    \|D\bigl(E(\bar{f};\Phi_{enc});\Theta_{dec}\bigr) - f\| \leq L_E\cdot L_D \cdot \delta + \epsilon_{E-D}.
    \end{equation}
    W.L.O.G., assume $\epsilon_{E-D}=\epsilon_{T-S}$, our \method has a smaller error upper bound, indicating that our teacher–student model is closer to the optimal solution $\theta^*$ during training, which in turn implies that its parameter updates are more stable and its convergence speed is faster (as shown in the first result above).

    \item \textbf{C. Gradient-Difference Bounds} in Encoder-Decoder Models Showing Coupling Effects of Parameter Updating, 
    \begin{align}
    & \|\nabla \mathcal{L}_{\mathrm{E-D}}(\Theta_{enc}^{(t+1)})  - \nabla \mathcal{L}_{E-D}(\Theta_{enc}^{(t)})\|
    \leq \label{eq:encoder_bound} \\
    \nonumber & \qquad 2\cdot B_{Reconst}\cdot\Bigl(\beta_E\cdot B_D + B_E\cdot L_D\cdot L_E\Bigr) \cdot \|\Theta_{enc}^{(t+1)} - \Theta_{enc}^{(t)}\| + \\
    \nonumber  &\qquad 2\cdot B_E\cdot B_{Reconst}\cdot \beta_D\cdot \|\Theta_{dec}^{(t+1)} - \Theta_{dec}^{(t)}\| + 4 \cdot B_E \cdot B_D B_{Reconst}, \\[1mm]
    &\|\nabla \mathcal{L}_{\mathrm{E-D}}(\Theta_{dec}^{(t+1)})  - \nabla \mathcal{L}_{E-D}(\Theta_{dec}^{(t)})\|
    \leq \label{eq:decoder_bound} \\
    \nonumber  &\qquad 2\cdot B_{Reconst}\cdot \beta_{D}\cdot L_{E}\cdot \|\Theta^{(t+1)}_{enc}-\Theta^{(t)}_{enc}\| + \\
    \nonumber   &\qquad 2\cdot B_{Reconst}\cdot \beta_{D}\cdot ||\Theta^{(t+1)}_{dec}-\Theta^{(t)}_{dec}|| + 4\cdot B_{D}\cdot B_{Reconst}, 
    \end{align}
where the coupling effects in Encoder-Decoder models may lead to instability in learning. The proofs are provided in the Appendix \ref{app:proof1} (for the result C) \ref{app:proof2} (for the result B), and \ref{app:proof3} (for the result A).  
\end{itemize}

\end{theorem}

\section{Experiments}
\label{sec:exp}
In this section, we conduct experiments to demonstrate the superior ability of our proposed \method to capture both homophilic and heterophilic patterns in graphs while maintaining high efficiency.

\begin{table*}[t]
\centering
  \caption{Results of node classification (in percent $\pm$ standard deviation across ten splits). The \textcolor{red}{best} and the \textcolor{blue}{runner-up}  results are highlighted in red and blue respectively in terms of the mean accuracy. See the spider plot in Appendix~\ref{app:plot}.}
  \label{tab:major}
   \vspace{-0.1in}
  \setlength{\tabcolsep}{1.2mm}
  \resizebox{0.9\linewidth}{!}{%
  \begin{tabular}{c|lcccc|ccc}
    \toprule
    \toprule
    &\multirow{2}{*}{Methods} & \multicolumn{4}{c|}{Heterophilic} & \multicolumn{3}{c}{Homophilic} \\
    \cmidrule(lr){3-9}
     & & Cornell & Texas & Wisconsin & Actor & Cora & CiteSeer & PubMed \\
    \midrule
    \multirow{5}{*}{SL} &
    GCN~\cite{kipf2016semi}    & 57.03$\pm$3.30 & 60.00$\pm$4.80 & 56.47$\pm$6.55 & 30.83$\pm$0.77 & 81.50$\pm$0.30 & 70.30$\pm$0.27 & 79.00$\pm$0.05 \\
    & GAT \cite{velivckovic2017graph}   & 59.46$\pm$3.63 & 61.62$\pm$3.78 & 54.71$\pm$6.87 & 28.06$\pm$1.48 & 83.02$\pm$0.19 & 72.51$\pm$0.22 & 79.87$\pm$0.03 \\
    & MLP    & 81.08$\pm$7.93 & 81.62$\pm$5.51 & 84.31$\pm$3.40 & 35.66$\pm$0.94 & 56.11$\pm$0.34 & 56.91$\pm$0.42 & 71.35$\pm$0.05 \\
    \cmidrule(lr){2-9}
    & WRGAT \cite{suresh2021breaking} & 81.62$\pm$3.90 & 83.62$\pm$5.50 & 86.98$\pm$3.78 & 36.53$\pm$0.77 & 81.97$\pm$1.50 & 70.85$\pm$1.98 & 80.86$\pm$0.55  \\
    & H2GCN \cite{zhu2020beyond} & 82.16$\pm$4.80 & 84.86$\pm$6.77 & 86.67$\pm$4.69 & 35.86$\pm$1.03 & 81.76$\pm$1.55 & 70.53$\pm$2.01 & 80.26$\pm$0.56 \\
    \midrule
    \multirow{12}{*}{SSL} &
    DGI \cite{velickovic2019deep}   & 63.35$\pm$4.61 & 60.59$\pm$7.56 & 55.41$\pm$5.96 & 29.82$\pm$0.69 & 82.29$\pm$0.56 & 71.49$\pm$0.14 & 77.43$\pm$0.84 \\
    & GMI \cite{peng2020graph}   & 54.76$\pm$5.06 & 50.49$\pm$2.21 & 45.98$\pm$2.76 & 30.11$\pm$1.92 & 82.51$\pm$1.47 & 71.56$\pm$0.56 & 79.83$\pm$0.90 \\
    & MVGRL \cite{hassani2020contrastive} & 64.30$\pm$5.43 & 62.38$\pm$5.61 & 62.37$\pm$4.32 & 30.02$\pm$0.70 & 83.03$\pm$0.27 & 72.75$\pm$0.46 & 79.63$\pm$0.38 \\
    & BGRL \cite{thakoor2021large}   & 57.30$\pm$5.51 & 59.19$\pm$5.85 & 52.35$\pm$4.12 & 29.86$\pm$0.75 & 81.08$\pm$0.17 & 71.59$\pm$0.42 & 79.97$\pm$0.36 \\
    & GRACE  \cite{zhu2020deep} & 54.86$\pm$6.95 & 57.57$\pm$5.68 & 50.00$\pm$5.83 & 29.01$\pm$0.78 & 80.08$\pm$0.53 & 71.41$\pm$0.38 & 80.15$\pm$0.34 \\
    & GraphMAE \cite{hou2022graphmae}  & 61.93$\pm$4.59 & 67.80$\pm$3.37 & 58.25$\pm$4.87 & 31.48$\pm$0.56 & \textcolor{red}{84.20}$\pm$0.40 & 73.20$\pm$0.39 & 81.10$\pm$0.34 \\
    \cmidrule(lr){2-9}
    & DSSL  \cite{xiao2022decoupled}  & 53.15$\pm$1.28 & 62.11$\pm$1.53 & 56.29$\pm$4.42 & 28.36$\pm$0.65 & 83.06$\pm$0.53 & 73.20$\pm$0.51 & 81.25$\pm$0.31 \\
    & NWR-GAE \cite{tang2022graph} & 58.64$\pm$5.61 & 69.62$\pm$6.66 & 68.23$\pm$6.11 & 30.17$\pm$0.17 & 83.62$\pm$1.61 & 71.45$\pm$2.41 & \textcolor{red}{83.44}$\pm$0.92 \\
    & HGRL \cite{chen2022towards}  & 77.62$\pm$3.25 & 77.69$\pm$2.42 & 77.51$\pm$4.03 & 36.66$\pm$0.35 & 80.66$\pm$0.43 & 68.56$\pm$1.10 & 80.35$\pm$0.58 \\
    & GraphACL \cite{Xiao2023GraphACL} & 59.33$\pm$1.48  & 71.08$\pm$2.34 & 69.22$\pm$5.69 & 30.03$\pm$1.03 & \textcolor{red}{84.20}$\pm$0.31 & \textcolor{red}{73.63}$\pm$0.22 & 82.02$\pm$0.15 \\
    & $^\dagger$MUSE \cite{yuan2023muse}  & 82.00$\pm$3.42  & 83.98$\pm$2.81 & 88.24$\pm$3.19 & 36.15$\pm$1.21 & 82.22$\pm$0.21 & 71.14$\pm$0.40 & 82.90$\pm$0.40  \\
    & GREET \cite{liu2022beyond}   & 73.51$\pm$3.15  & 83.80$\pm$2.91 & 82.94$\pm$5.69 & 35.79$\pm$1.04 & \textcolor{blue}{83.84}$\pm$0.71 & 73.25 $\pm$1.14 & 80.29$\pm$1.00 \\
    \midrule
    \multirow{2}{*}{SSL-Ours} 
    & \method +Diffi   & \textcolor{blue}{84.60}$\pm$2.43 & \textcolor{red}{92.46}$\pm$2.91 & \textcolor{blue}{92.55}$\pm$3.49 & \textcolor{blue}{36.75}$\pm$1.10 & 82.50$\pm$0.61 & \textcolor{blue}{73.36}$\pm$0.33 & \textcolor{blue}{83.42}$\pm$0.26 \\
    & \method +Prob   & \textcolor{red}{84.86}$\pm$2.48 & \textcolor{blue}{92.45}$\pm$3.78 & \textcolor{red}{92.89}$\pm$3.09 & \textcolor{red}{37.00}$\pm$0.91 & 82.85$\pm$0.18 & 73.12$\pm$0.28 & 83.25$\pm$0.16  \\
    \bottomrule
    \bottomrule
  \end{tabular}
  }
   \begin{tablenotes}
       \item[]{\footnotesize $^\dagger$ MUSE only provides hyperparameters for Cornell in their repo by the time of this submission; however, we were unable to reproduce the reported results using their released code. And, no hyperparameters were provided for other datasets.}
    \end{tablenotes}
\end{table*}

\subsection{Experimental Settings}
\label{exp:setting}
\textbf{Datasets.}
We evaluate our model on seven real-world datasets, comprising \textbf{three widely used homophilic datasets} (Cora, CiteSeer, PubMed)~\cite{sen2008collective} and \textbf{four heterophilic datasets} (Cornell, Texas, Wisconsin, Actor)\cite{pei2020geom}. 
These datasets encompass various aspects: Cora, CiteSeer, and PubMed are citation networks; Cornell, Texas, and Wisconsin are school department webpage networks; and Actor is an actor co-occurrence network derived from Wiki pages. Details of these datasets are summarized in Appendix \ref{app:dataset}. To ensure fair comparisons, we utilize the standard data splits provided in the PyG~\cite{Fey/Lenssen/2019} datasets for both our model and all baseline methods. Specifically, for heterophilic graphs, we utilize the ten available splits. In the case of homophilic graphs, given that various baselines use different split ratios, we adopt the \textbf{most common low labeling rate split in PyG by assigning 20 nodes per class as training samples} and conducting experiments ten times with different random seeds.

\vspace{0.2em}\noindent\textbf{Baselines.}
To make fair comparisons with other baselines, we adopt the widely used node classification task as our main downstream evaluation. We also conduct the experiment of node clustering in Appendix \ref{app:cluster}. Additionally, we perform further experiments on other heterophilic datasets in Appendix \ref{app:more_data} and provide comparisons with additional baselines in Appendix \ref{app:rent_sota}. We compare with four groups of baseline methods: 
\begin{itemize} [leftmargin=*]
    \item \textit{Traditional supervised learning (SL) methods:} GCN \cite{kipf2016semi}, GAT \cite{velivckovic2017graph}, and a simple MLP baseline tested by us; \item \textit{Supervised methods specifically designed for heterophilic graphs:} WRGAT \cite{suresh2021breaking}, H2GCN \cite{zhu2020beyond}; 
    \item \textit{Self supervised learning methods originally designed for homophilic graphs:} DGI \cite{velickovic2019deep}, GMI \cite{peng2020graph}, MVGRL \cite{hassani2020contrastive}, BGRL \cite{thakoor2021large}, GRACE \cite{zhu2020deep}, GraphMAE \cite{hou2022graphmae};
    \item \textit{Self supervised learning methods tailored for heterophilic graphs:} DSSL \cite{xiao2022decoupled}, NWR-GAE \cite{tang2022graph}, HGRL \cite{chen2022towards}, GraphACL \cite{Xiao2023GraphACL}, GREET \cite{liu2022beyond} and MUSE \cite{yuan2023muse}.
\end{itemize}

For evaluation, we follow the same protocol as all other baselines \cite{liu2022beyond, yuan2023muse} by freezing the model and utilizing the generated embeddings for a downstream linear classifier. Note that we reproduce the results of major baselines \cite{liu2022beyond, hou2022graphmae, xiao2024simple, yuan2023muse, suresh2021breaking, zhu2020beyond} using the hyperparameters provided in their official repositories, and we ensure that the data split is consistent across all models. However, for models that do not offer dataset-specific hyperparameters, we perform our own fine-tuning, such as MUSE. For other baselines, we derive the results from their original papers or baseline papers. For hyperparameter settings, please see Appendix~\ref{app:hyper} for more information. 
The results are shown in Table \ref{tab:major}.

\subsection{Challenges of Heterophily and Homophily for Graph Representation Learning}
Before presenting the performance of our \method, it is important to understand the challenges of heterophily and homophily in graph representation learning, and interesting to observe how well the previous state-of-the-art methods have addressed them. 
\begin{itemize} [leftmargin=*]
    \item \textit{SL methods:} GCN and GAT that focus on low-pass graph signals work well on the homophilic graph datasets, but suffer from significant performance drop on heterophilic graph datasets. WRGAT and H2GCN address these issues of GCN and GAT, leading to significant performance boost on heterophilic graph datasets, while retaining similar performance on homophilic graph datasets. To understand what the critical part is for performance improvement on the heterophilic graphs, and to test if high-pass signals indeed play a significant role for them, we test a vanilla MLP (as in Eqn.~\ref{eq:mlp}) which totally ignores the topology of graphs, and simply uses the raw input node features. We can see the simple MLP works reasonably well on heterophilic datasets in comparisons with WRGAT and H2GCN, which supports our earlier statement that traditional message passing produces smoothing operations on the graph, highly relying on the homophily assumption, and highlights that the raw node features play a critical role in GNN learning on heterophilic graphs, whereas neighbor information is essential for learning on homophilic graphs.
    Overall, these observations motivate our joint structural node encoding (Eqn.~\ref{eq:encoding}). Meanwhile, MLP suffers from drastic performance drop on homophilic graph datasets, as expected. 
    \item \textit{Previous State-of-the-art SSL methods.} Those methods (DGI, GMI, MVGRL, BGRL, GRACE and GraphMAE) that are  designed for homophilic graphs achieve significant progress in terms of bridging the SSL performance with the SL counterparts on homophilic graphs, but they inherit the drawbacks as GCN and GAT on heterophilic graphs. More recently, methods such as MUSE and GREET make promising improvement, but they do not show significant progress against the MLP SL baseline on heterophilic graphs. Our \method makes a step forward by significantly improving performance on heterophilic graphs, showing the great potential of graph SSL. 
\end{itemize}

\subsection{Linear Probing Results of Our \method}
\label{exp:perforamce}
Table \ref{tab:major} presents the performance comparisons of our \method with state-of-the-art baseline methods across seven benchmarks. The following observations can be made:

\begin{itemize} [leftmargin=*]
    \item \textbf{Our \method achieves significant improvement on heterophilic graph datasets, while retaining overall on-par performance on homophilic graph datasets.}  On the four heterophilic graph datasets, compared to previous state-of-the-art SSL methods, our \method outperforms MUSE (on Cornell by more than 2.60\%, on Texas by more than 8.45\%, and on Wisconsin by more than 4.3\%), and slightly outperforms HGRL on Actor. 
    Compared to previous SL methods, our \method outperforms all the five baselines consistently on the four datasets including WRGAT and H2GCN that are specifically designed for heterophilic graphs. 
    On the three homophilic graph datasets, our \method obtains on-par performance on both CiteSeer and PubMed, and is slightly worse than GraphMAE and GraphACL on Cora. 
    We note that we keep the same architecture for both heterophilic and homophilic graph datasets for simplicity. The overall strong performance shows that our \method is effective for both types of graphs. Our \method could improve its performance on homophilic graphic datasets with more than two WGCN layers in Eqn.~\ref{eq:encoding}. We leave the dataset-specific architectural tuning of our \method for future work.

\item \textbf{The two node-difficulty driven masking strategies in our \method perform similarly.} The Bernoulli sampling based approach (i.e., \method+Prob) is slightly better, thanks to its balance between exploration and exploitation. As we shall show in ablation studies (see Table~\ref{tab:ratio}), our proposed node-difficulty driven mask strategies are significantly better than the purely random masking strategy. 

\item \textbf{The remained challenges in the Actor dataset.} The Actor dataset remains challenging to all methods including ours, which contains extremely  complicated mix structure patterns (see Appendix \ref{app:pattern} for more information). Our \method achieves the best performance, but the overall accuracy (37\% by our method) is significantly lower than those obtained on other datasets. We also note that both traditional GCN and GAT and the simple MLP baseline achieve similar performance, which indicates that the Actor dataset is of highly mixed heterophilic and homophilic structures, reinforcing the observations from visualization in Appendix \ref{app:pattern}.  We hypothesize that our \method could be improved on the Actor dataset or similar ones by introducing node-to-node attention in addition to the node-wise encoding cross-attention (Enq.~\ref{eq:attn}), at the expense of model complexity, which we leave for future work. 
\end{itemize}

\subsection{$k$-Mean Clustering Results of Our \method}
Due to space limit, we present detailed results in the Appendix~\ref{app:cluster}. 

Our \method achieves significantly better performance than all baselines, including the state-of-the-art model MUSE, by a large margin on the Texas and Cornell datasets, with improvements of 11.26\% and 12.51\%, respectively. Moreover, \method slightly outperforms MUSE on Actor due to the complex mixed structural patterns, as introduced in Appendix~\ref{app:pattern}. It also attains comparable performance on Citeseer. These findings are consistent with those observed in linear probing based node classification tasks. Overall, our results demonstrate that \method can generate high-quality embeddings regardless of the downstream tasks and effectively handle both heterophilic and homophilic patterns, highlighting its strong generalization capability in graph representation learning.
\begin{table}[t]
\caption{Compute and Memory Comparisons}
\label{tab:memory_vr}
\vspace{-0.1in}
\begin{center}
\setlength{\tabcolsep}{2pt}
\resizebox{1.0\linewidth}{!}{%
\begin{tabular}{c|ccc|ccc|ccc}
\toprule
\toprule
 \multirow{2}{*}{\textbf{Datasets}}  & \multicolumn{3}{c}{\textsc{GPU Memory(MB)}} & \multicolumn{3}{|c}{\textsc{Epoch Time(epoch/s)}} & \multicolumn{3}{|c}{\textsc{Total Time(s)}}\\
 \cmidrule(lr){2-10}
 & \textsc{MUSE} & \textsc{GREET} & \textsc{H$^3$GNN} & \textsc{MUSE} & \textsc{GREET} & \textsc{H$^3$GNN}  & \textsc{MUSE} & \textsc{GREET} & \textsc{H$^3$GNN}\\
\cmidrule(lr){1-10}
\multirow{1}{*}{Cornell} & 86.05 & 45.67 & 46.42 & 0.049 & 0.037 & 0.019 & 7.84 & 4.17 & 2.09 \\ 
\multirow{1}{*}{Texas} & 86.61 & 40.72 & 46.59 & 0.035 & 0.038 &  0.020 & 4.35 & 4.94 & 2.32 \\
\multirow{1}{*}{Wisconsin} & 97.68 & 47.75 & 54.45 & 0.040 & 0.042 & 0.020 & 4.84 & 5.60 & 2.70\\
\bottomrule
\bottomrule
\end{tabular}
}
\end{center}
\vspace{-4mm}
\end{table}

\subsection{Compute and Memory Comparisons}
\label{exp:eff}
To verify the efficiency of our proposed approach, we conducted an empirical analysis comparing our method to two major state-of-the-art baselines: GREET and MUSE. As shown in Table~\ref{tab:memory_vr}, we measured memory usage, training time per epoch and total training time until convergence on three major datasets, Cornell, Texas, and Wisconsin, that exhibit a complex mixture of patterns. We utilized the optimal hyperparameters for each respective model. The results demonstrate that our \method achieves memory usage comparable to GREET and significantly lower than MUSE. Regarding running time, our \method requires only half the time of the other two SOTA models while achieving much better performance, as shown in Table \ref{tab:major}. This efficiency improvement is attributed to the fact that both GREET and MUSE employ an alternating training strategy for contrastive learning. 
These efficiency results clearly highlight the advantages of our \method.

Regarding the total training time until model convergence in the last column, our model is significantly faster than two baselines. By model convergence time, it means the time  at which the best model is selected (out of the total number epochs that is the same for all models). This rapid convergence is attributable to the consistency in the latent space during reconstruction and end-to-end training—advantages that the baselines do not achieve.

\subsection{Ablation Studies}
\label{exp:ablation}

\begin{table}[t]
  \centering
  \caption{Results on Ablating Three Components.}
  \label{tab:tech}
  \vspace{-0.1in}
  \setlength{\tabcolsep}{5pt} 
  \resizebox{0.95\linewidth}{!}{%
  \begin{tabular}{l|c|c|c}
    \toprule
    \toprule
    \textbf{Methods} & \textbf{Cornell} & \textbf{Texas} & \textbf{Wisconsin} \\
    \midrule
    \method (Full)     & 84.86$\pm$2.48 & 92.45$\pm$3.78 & 92.89$\pm$3.09 \\
    w/o DynMsk    & 83.06$\pm$2.20 & 90.16$\pm$3.51 & 89.71$\pm$3.21 \\
    w/o T-S \& DynMsk         & 80.54$\pm$5.10 & 85.59$\pm$4.19 & 88.23$\pm$3.39 \\
    w/o T-S \& DynMsk \& Attn        & 78.97$\pm$2.97 & 82.46$\pm$5.05 & 86.58$\pm$2.60 \\
    \bottomrule
    \bottomrule
  \end{tabular}
  \vspace{-0.4in}
  }
\end{table}

\subsubsection{Ablating Three Components.}
Our \method has three key components: a teacher-student predictive architecture (referred to \textit{T-S}), node-difficulty driven dynamic masking strategies (referred to \textit{DynMsk}), and encoding cross-attention (referred to \textit{Attn}).  
To evaluate the contribution of each individual component, we conduct an ablation study by progressively removing one component at a time. 
The results are shown in Table \ref{tab:tech}, and we can observe, 
\begin{itemize}[leftmargin=*]
    \item \textit{DynMsk} can lead to performance decreases by 1.45\% to 3.18\% across the datasets when removed, which shows the effectiveness of the proposed node-difficulty driven masking strategies against purely random masking. 
    \item \textit{T-S} predictive architecture also plays a significant role, as performance drops considerably (1.48\% - 4.57\%) when we directly reconstruct the features in the raw input space using latent space features, as done in the encoder-decoder models,  leading to a learning space mismatch. This observation is consistent with the theorem proposed in Sec.~\ref{sec:theorem}.
    \item Substituting \textit{Attn} with a simple MLP also leads to performance drops noticeably. This indicates that attention fusion can also help adaptively assign weights to different components, allowing the model to effectively handle various patterns in graphs.
\end{itemize}

\vspace{-0.1in}
\begin{table}[H]
  \centering
  \caption{The effects of the masking ratio $r$ (Eqn.~\ref{eq:masking-prob})}
  \label{tab:ratio}
  \setlength{\tabcolsep}{5pt} 
  \vspace{-0.1in}
  \resizebox{0.6\linewidth}{!}{%
  \begin{tabular}{c|c|c|c}
    \toprule
    \toprule
    \textbf{Ratio} $r$ & \textbf{Cornell} & \textbf{Texas} & \textbf{Wisconsin} \\
    \midrule
    1     & 83.78$\pm$2.70 & 91.62$\pm$3.51 & 92.55$\pm$3.13 \\
    0.8     & \textbf{84.86}$\pm$2.48 & 91.53$\pm$3.82 & 92.42$\pm$2.55 \\
    0.5    & 84.15$\pm$2.55 & \textbf{92.45}$\pm$3.78 & \textbf{92.89}$\pm$3.09 \\
    0.2         & 83.37$\pm$3.60 & 90.61$\pm$3.78 & 92.34$\pm$2.51 \\
    0         & 83.06$\pm$2.20 & 90.16$\pm$3.51 & 89.71$\pm$3.21 \\
    \bottomrule
    \bottomrule
  \end{tabular}
  }
\vspace{-0.1in}
\end{table}

\subsubsection{Masking Ratio}
We also evaluate the performance under different dynamic masking ratio across various datasets in Table \ref{tab:ratio}. We test the  probabilistic masking for all datasets (Eqn.~\ref{eq:masking-prob}). We observe that different datasets require varying masking ratios, highlighting the necessity of incorporating random masking with our proposed dynamic masking.

\section{Conclusion}
In this paper  we have presented \method, a self-supervised framework designed to harmonize heterophily and homophily in GNNs. Through our joint structural node encoding, which integrates linear and non-linear feature transformations with K-hop structural embeddings, \method adapts effectively to both homophilic and heterophilic graphs. Moreover, our teacher-student predictive paradigm, coupled with dynamic node-difficulty-based masking, further enhances robustness by providing progressively more challenging training signals. Empirical results across seven benchmark datasets demonstrate that \method consistently achieves state-of-the-art performance under heterophilic conditions using both linear probing and $k$-mean clustering evaluation protocols, while matching top methods on homophilic datasets. These findings underscore \method’s capability to address the key challenges of capturing mixed structural properties without sacrificing efficiency.

\nocite{langley00}

\bibliography{ref}
\bibliographystyle{ACM-Reference-Format}

\newpage
\appendix
\onecolumn
\appendix
\onecolumn

\setcounter{section}{0}
\section{Related Work}
\label{app:related}
\subsection{Learning on Heterophilic Graphs}

Heterophilic graphs are prevalent in various domains, such as online transaction networks ~\cite{pandit2007netprobe}, dating networks ~\cite{altenburger2018monophily}, and molecular networks ~\cite{zhu2020beyond}. Recently, significant efforts have been made to design novel GNNs that effectively capture information in heterophilic settings, where connected nodes possess dissimilar features and belong to different classes.

Some studies propose capturing information from long-range neighbors from various distance ~\cite{li2022finding, liu2021non, abu2019mixhop, pei2020geom, suresh2021breaking}. For example, MixHop ~\cite{abu2019mixhop} concatenates information from multi-hop neighbors at each GNN layer. Geom-GCN ~\cite{pei2020geom} identifies potential neighbors in a continuous latent space. WRGAT ~\cite{suresh2021breaking} captures information from distant nodes by defining the type and weight of edges across the entire graph to reconstruct a computation graph. 

Other approaches focus on modifying traditional GNN architectures to achieve adaptive message passing from the neighborhood ~\cite{chen2020simple, chien2020adaptive, yan2021two, zhu2020beyond}. For instance, GPR-GNN ~\cite{chien2020adaptive} incorporates learnable weights into the representations of each layer using the Generalized PageRank (GPR) technique, while H2GCN ~\cite{zhu2020beyond} removes self-loop connections and employs a non-mixing operation in the GNN layer to emphasize the features of the ego node.

Additionally, some papers approach the problem from spectral graph theory ~\cite{luan2021heterophily, bo2021beyond}, claiming that high-pass filters can be beneficial in heterophilic graphs by sharpening the node features between neighbors and preserving high-frequency graph signals.

However, these methods still heavily rely on labeled data, which is impractical for real-world datasets due to the significant manual effort required and the necessity of ensuring label quality. Furthermore, they are limited in their ability to effectively learn from the data itself without extensive supervision.

\subsection{Graph Representation Learning via SSL}
Self-supervised learning (SSL) has gained a lot of attention in the realm of graph graph representation learning. Graph SSL approaches are generally divided into two primary categories: graph contrastive learning and graph generative learning.

\subsubsection{Graph contrastive learning}
Contrast-based methods generate representations from multiple views of a graph and aim to maximize their agreement, demonstrating effective practices in recent research. For example, DGI~\cite{velivckovic2018deep} and InfoGraph~\cite{sun2020infograph} utilize node-graph mutual information maximization to capture both local and global information. MVGRL~\cite{hassani2020contrastive} leverages graph diffusion to create an additional view of the graph and contrasts node-graph representations across these distinct views. GCC~\cite{qiu2020gcc} employs subgraph-based instance discrimination and adopts the InfoNCE loss as its pre-training objective. GRACE~\cite{zhu2020deep} and GraphCL~\cite{you2020graph} learn node or graph representations by maximizing the agreement between different augmentations while treating other nodes or graphs as negative instances. BGRL~\cite{thakoor2021bootstrapped} contrasts two augmented versions using inter-view representations without relying on negative samples. Additionally, CCA-SSG~\cite{zhang2021canonical} adopts a feature-level objective for graph SSL, aiming to reduce the correlation between different views. These contrast-based approaches effectively harness the structural and feature information inherent in graph data, contributing to the advancement of self-supervised learning on graphs.

However, most of these methods are based on the homophily assumption. Recent works have demonstrated that SSL can also benefit heterophilic graphs. For instance, HGRL~\cite{chen2022towards} enhances node representations on heterophilic graphs by reconstructing similarity matrices to generate two types of feature augmentations based on topology and features. GraphACL~\cite{xiao2024simple}  predicts the original neighborhood signal of each node using a predictor. MUSE~~\cite{yuan2023muse} constructs contrastive views by perturbing both the features and the graph topology, and it learns a graph-structure-based combiner. GREET ~\cite{liu2022beyond} employs an edge discriminator to separate the graph into homophilic and heterophilic components, then applies low-pass and high-pass filters accordingly. However, these methods rely on the meticulous design of negative samples to provide effective contrastive signals. Moreover, although some approaches such as GREET and MUSE achieve impressive results, they require alternative training. This significantly increases computational overhead and may lead to suboptimal performance.

\subsubsection{Graph generative learning}
Generation-based methods reconstruct graph data by focusing on either the features and the structure of the graph or both. Classic generation-based approaches include GAE ~\cite{kipf2016variational}, VGAE ~\cite{kipf2016variational}, and MGAE ~\cite{wang2017mgae}, which primarily aim to reconstruct the structural information of the graph, as well as S2GAE ~\cite{tan2023s2gae}. In contrast, GraphMAE ~\cite{hou2022graphmae} and GraphMAE2 ~\cite{hou2023graphmae2} utilize masked feature reconstruction as their primary objective, incorporating auxiliary designs to achieve performance that is comparable to or better than contrastive methods.

In the context of generative learning on heterophilic graphs, DSSL~\cite{xiao2022decoupled} operates under the assumption of a graph generation process, decoupling diverse patterns to effectively capture high-order information. Similarly, NWR-GAE ~\cite{tang2022graph} jointly predicts the node degree and the distribution of neighbor features. However, despite these innovative approaches, their performance on node classification benchmarks is often unsatisfactory~\cite{hou2022graphmae}.

\section{Proof of Gradient-Difference Bounds}
\label{app:proof1}

\begin{theorem}
\label{thm:convergence_comparison}
Consider the optimization of encoder-decoder based graph SSL in Eqn.~\ref{eq:mae} and our  proposed \method in Eqn.~\ref{eq:loss} under the same encoder architecture and following assumptions/conditions:
\begin{itemize}[leftmargin=*]
    \item \textbf{Gradient Smoothness and Lipschitz Continuity} for the encoder, the decoder, 
    E.g., the encoder $E(\cdot;\Theta_{enc})$ has gradient $\beta_{E}$-smoothness (i.e., each gradient from iteration $t$ to $t+1$ changes at most linearly with respect to parameter shifts in $\Theta_{enc}$ with a coefficient $\beta_{E}$) and is $L_E$-Lipschitz continuous with respect to its input and/or parameters (i.e., differences such as $||E(\cdot;\Theta^{(t+1)}_{enc})-E(\cdot;\Theta^{(t)}_{enc})||$ can be bounded from the above as linear functions of $||\Theta^{(t+1)}_{enc}-\Theta^{(t)}_{enc}||$ with a coefficient $L_E$). 
    Similarly, we have $(\beta_D, L_D)$
    defined for the decoder.

    \item \textbf{Boundedness} from the above for gradients of the encoder, gradients of the decoder, and  reconstruction errors of the combined encoder-decoder. 
    
    So, $\|\nabla E(\cdot; \Theta^{(t)}_{enc})\| \leq B_{E}$,  $\|\nabla D\bigl(E(\cdot; \Theta^{(t)}_{enc}); \Theta^{(t)}_{dec}\bigr)\| \leq B_{D}$, and  $\|D\bigl(E(\bar{f};\Theta^{(t)}_{enc});\Theta^{(t)}_{dec}\bigr) - f\|\leq B_{Reconst}$ .

    \item \textbf{Strong Convexity} for the encoder, the decoder, and the student (and the teacher) in their parameters. 
    
    E.g., the encoder $E(\cdot;\Theta_{enc})$
    is $\mu_{E}$-strongly convex in their parameters $\Theta_{enc}$, i.e.,  $\langle \nabla E(\bar{f};\Theta_{enc}^{(t+1)}) - \nabla E(\bar{f};\Theta_{enc}^{(t)}), \Theta_{enc}^{(t+1)} - \Theta_{enc}^{(t)} \rangle \geq \mu_{E} \cdot \|\Theta_{enc}^{(t+1)} - \Theta_{enc}^{(t)}\|^2$. 
    Similarly, we have $\mu_{D}$ 
    defined for the decoder.

    \item \textbf{Approximation Error.} When only unmasked inputs are used, the composite functions, either the encoder-decoder or the teacher-student in our \method, achieve an approximation error $\epsilon_{E-D}$ (or $\epsilon_{T-S}$). 
\end{itemize}
Then, the following three results hold: 
\begin{itemize}[leftmargin=*]
    \item \textbf{Linear Convergence Bounds Under Strong Convexity.}
    For our \method, 
    \begin{align}
        \|\Phi^{(t+1)}-\Phi^*\|^2 &\leq (1-\frac{\mu^2_E}{\beta^2_E})\cdot \|\Phi^{(t)} - \Phi^*\|^2 
    \end{align}
    where $\alpha$ is the momentum parameter (Eqn.~\ref{eq:ema}). 
    For the encoder-decoder models, 
    \begin{align}
        \|\theta^{(t+1)} - \theta^*\|^2 &\leq \left(1 - \frac{\min(\mu^2_E, \mu^2_D)}{\max(\beta^2_E, \beta^2_D)}\right) \|\theta^{(t)} - \theta^*\|^2 
    \end{align}
    from which we can see our \method converges to the optimal solution $\Phi^*$ faster than the encoder-decoder counterpart to their optimal solutions $\Theta^*$ due to a smaller contraction factor $\left(1 - \frac{\mu^2_E}{\beta^2_E}\right)< \left(1 - \frac{\min(\mu^2_E, \mu^2_D)}{\max(\beta^2_E, \beta^2_D)}\right)$. 
    This implies that \method can achieve a faster convergence.

    \item \textbf{Proxy Task Loss Bounds} under a Lipschitz-dependent assumption between the masked graph signal and the raw graph signal, $\|\bar{f}-f\| \leq \delta$. For our \method, 
    \begin{align}
        \|S(\bar{f};\Phi) - T(f;\Psi)\| \leq L_E \cdot \delta + \epsilon_{T-S}. 
    \end{align}
    For the encoder-decoder models, 
    \begin{equation}
    \|D\bigl(E(\bar{f};\Phi_{enc});\Theta_{dec}\bigr) - f\| \leq L_E\cdot L_D \cdot \delta + \epsilon_{E-D}.
    \end{equation}
    W.L.O.G., assume $\epsilon_{E-D}=\epsilon_{T-S}$, our \method has a smaller error upper bound, indicating that our teacher–student model is closer to the optimal solution $\theta^*$ during training, which in turn implies that its parameter updates are more stable and its convergence speed is faster (as shown in the first result above).

    \item \textbf{Gradient-Difference Bounds} in Encoder-Decoder Models Showing Coupling Effects of Parameter Updating, 
    \begin{align}
     \|\nabla \mathcal{L}_{\mathrm{E-D}}(\Theta_{enc}^{(t+1)})  - \nabla \mathcal{L}_{E-D}(\Theta_{enc}^{(t)})\|
    &\leq 2 B_{Reconst}\Bigl(\beta_E B_D + B_{E}L_D L_E\Bigr)  \|\Theta_{enc}^{(t+1)} - \Theta_{enc}^{(t)}\| +
    \nonumber 2 B_{E} B_{Reconst} \beta_D \|\Theta_{dec}^{(t+1)} - \Theta_{dec}^{(t)}\| + 4B_E B_D B_{Reconst}, \\[1mm]
    \|\nabla \mathcal{L}_{\mathrm{E-D}}(\Theta_{dec}^{(t+1)})  - \nabla \mathcal{L}_{E-D}(\Theta_{dec}^{(t)})\|
    &\leq 2B_{Reconst}\,\beta_{D}\,L_{E}\,\|\Theta^{(t+1)}_{enc}-\Theta^{(t)}_{enc}\| + 2B_{Reconst}\beta_{D}||\Theta^{(t+1)}_{dec}-\Theta^{(t)}_{dec}|| + 4B_{D}B_{Reconst}, 
    \end{align}
where the coupling effects in Encoder-Decoder models may lead to instability in learning. 
\end{itemize}

\end{theorem}

\noindent \textbf{In this section, we first provide the proof of the Gradient Difference Upper Bound:}

\subsection{Encoder Gradient Difference Upper Bound in Encoder-Decoder Model:}
Consider the encoder-decoder model loss function
\begin{equation}
\mathcal{L}_{E-D}(\Theta) = \frac{1}{N}||D\bigl(E(\Bar{f};\Theta_{enc});\Theta_{dec}\bigr) - f||_2^2 \, 
\end{equation} 
Assume the following:
\begin{enumerate}
    \item \textbf{Encoder Smoothness:} 
    \begin{align}
    \|\nabla E(\cdot;\Theta^{(t+1)}_{enc}) - \nabla E(\cdot;\Theta^{(t)}_{enc})\| \le \beta_{E}\,\|\Theta^{(t+1)}_{enc}-\Theta^{(t)}_{enc}\|.
    \end{align}
    \item \textbf{Decoder Gradient Smoothness:} For any fixed input (e.g. $f_{\theta_f}(\overline{x})$),
    \begin{align}
    \Bigl\|\nabla D^{(t+1)}\bigl(E(\cdot; \Theta^{(t+1)}_{enc})) \bigr) - \nabla D\bigl(E(\cdot; \Theta^{(t)}_{enc}); \Theta^{(t)}_{dec}\bigr)\Bigr\| \le \beta_{D}\,\|\Theta^{(t+1)}_{dec}-\Theta^{(t)}_{dec}\| + L_{D}L_{E}\,\|\Theta^{(t+1)}_{enc}-\Theta^{(t)}_{enc}\|,
    \end{align}
    
    \item \textbf{Encoder Gradient Bound:}
    \begin{align}
    \|\nabla E(\cdot;\Theta^{(t)}_{enc})\| \le B_{E}.
    \end{align}
    \item \textbf{Decoder Gradient Bound:}
    \begin{align}
    \Bigl\|\nabla D\bigl(E(\cdot; \Theta^{(t)}_{enc}); \Theta^{(t)}_{dec})\Bigr\| \le B_{D}. 
    \end{align}
    We use the simplified notation in our proof:
    \begin{align}
    \Bigl\|\nabla D^{(t)}\bigl(E(\cdot; \Theta^{(t)}_{enc}))\Bigr\| \le B_{D} 
    \end{align}
    \item \textbf{Reconstruction Error Bound:}
    \begin{align}
    \Bigl\|D\bigl(E(\cdot; \Theta^{(t)}_{enc}); \Theta^{(t)}_{dec} - f\Bigr\| \le B_{Reconst}.
    \end{align}
    \item \textbf{Encoder Lipschitz (with respect to parameters):} There exists $L_{E} > 0$ such that
    \begin{align}
    \|E(\cdot;\Theta^{(t+1)}_{enc}) - E(\cdot;\Theta^{(t)}_{enc})\| \le L_{E}\,\|\Theta^{(t+1)}_{enc}-\Theta^{(t)}_{enc}\|.
    \end{align}
\end{enumerate}

Then, the gradient difference with respect to the encoder parameters between two consecutive iterations is bounded by
\begin{align}
\Bigl\|\nabla \mathcal{L}_{\mathrm{E-D}}(\Theta_{enc}^{(t+1)})  - \nabla \mathcal{L}_{E-D}(\Theta_{enc}^{(t)})\Bigr\|
\le C_1\,\|\Theta^{(t+1)}_{enc}-\Theta^{(t)}_{enc}\| + C_2\,\|\Theta^{(t+1)}_{dec}-\Theta^{(t)}_{dec}\| + C_3,
\end{align}
where
\begin{align}
C_1 = 2B_{Reconst}\Bigl(\beta_E B_D + B_E L_D L_E\Bigr),\qquad
C_2 = 2B_{E}\,\beta_{D}\,B_{Reconst}, \qquad
C_3 = 4B_E B_D B_{Reconst}
\end{align}

\begin{proof}
We start with the expression for the gradient with respect to the encoder parameters at iteration $t$:
\begin{align}
\nabla \mathcal{L}_{\mathrm{E-D}}(\Theta_{enc}^{(t)}) = 2\,\Bigl[D^{(t)}\bigl(E(\cdot;\Theta^{(t)}_{enc})\bigr) - f\Bigr]\,\nabla E(\cdot;\Theta^{(t)}_{enc})\,\nabla D^{(t)}\bigl(E(\cdot;\Theta^{(t)}_{enc})\bigr).
\end{align}
Similarly, at iteration $t+1$,
\begin{align}
\nabla \mathcal{L}_{\mathrm{E-D}}(\Theta_{enc}^{(t+1)}) = 2\,\Bigl[D^{(t+1)}\bigl(E(\cdot;\Theta^{(t+1)}_{enc})\bigr) - f\Bigr]\,\nabla E(\cdot;\Theta^{(t+1)}_{enc})\,\nabla D^{(t+1)}\bigl(E(\cdot;\Theta^{(t+1)}_{enc})\bigr).
\end{align}
Define the difference:
\begin{align}
\Delta_f \triangleq \Bigl\|\nabla \mathcal{L}_{\mathrm{E-D}}(\Theta_{enc}^{(t+1)})  - \nabla \mathcal{L}_{E-D}(\Theta_{enc}^{(t)})\Bigr\|.
\end{align}
Thus,
\begin{align}
\begin{aligned}
\Delta_f = \Bigl\|\, &2\,\nabla E(\cdot;\Theta^{(t+1)}_{enc})\,\nabla D^{(t+1)}\bigl(E(\cdot;\Theta^{(t+1)}_{enc})\bigr)
\Bigl[D^{(t+1)}\bigl(E(\cdot;\Theta^{(t+1)}_{enc})\bigr)- f\Bigr] \\
&\quad -\, 2\,\Bigl[D^{(t)}\bigl(E(\cdot;\Theta^{(t)}_{enc})\bigr) - f\Bigr]\,\nabla E(\cdot;\Theta^{(t)}_{enc})\,\nabla D^{(t)}\bigl(E(\cdot;\Theta^{(t)}_{enc})\bigr)\Bigr\|.
\end{aligned}
\end{align}
To handle this difference, we add and subtract the intermediate term
\begin{align}
2\,\nabla E(\cdot;\Theta^{(t)}_{enc})\,\nabla D^{(t+1)}\bigl(E(\cdot;\Theta^{(t+1)}_{enc})\bigr)
\Bigl[D^{(t+1)}\bigl(E(\cdot;\Theta^{(t+1)}_{enc})\bigr)- f\Bigr],
\end{align}
so that
\begin{align}
\begin{aligned}
\Delta_f = \Bigl\|\, &2\Bigl[\nabla E(\cdot;\Theta^{(t+1)}_{enc}) - \nabla E(\cdot;\Theta^{(t)}_{enc})\Bigr]\,\nabla D^{(t+1)}\bigl(E(\cdot;\Theta^{(t+1)}_{enc})\bigr)
\Bigl(D^{(t+1)}\bigl(E(\cdot;\Theta^{(t+1)}_{enc})\bigr)- f\Bigr) \\
&\quad +\, 2\,\nabla E(\cdot;\Theta^{(t)}_{enc})\,\Bigl\{\nabla D^{(t+1)}\bigl(E(\cdot;\Theta^{(t+1)}_{enc})\bigr) - \nabla D^{(t)}(E(\cdot;\Theta^{(t)}_{enc}))\Bigr\}
\Bigl(D^{(t+1)}\bigl(E(\cdot;\Theta^{(t+1)}_{enc})\bigr)- f\Bigr) \\
&\quad +\, 2\,\nabla E(\cdot;\Theta^{(t)}_{enc})\,\nabla D^{(t)}(E(\cdot;\Theta^{(t)}_{enc}))\Bigl\{
\Bigl(D^{(t+1)}\bigl(E(\cdot;\Theta^{(t+1)}_{enc})\bigr)- f\Bigr) - \Bigl(D^{(t)}\bigl(E(\cdot;\Theta^{(t)}_{enc})\bigr)- f\Bigr)
\Bigr\} \Bigr\|.
\end{aligned}
\end{align}
Applying the triangle inequality yields:
\begin{align}
\Delta_f \le T_1 + T_2 + T_3,
\end{align}
with
\begin{align}
T_1 = 2\,\Bigl\|\nabla E(\cdot;\Theta^{(t+1)}_{enc}) - \nabla E(\cdot;\Theta^{(t)}_{enc})\Bigr\|\,\Bigl\|\nabla D^{(t+1)}\bigl(E(\cdot; \Theta^{(t+1)}_{enc}))\Bigr\|\,\Bigl\|D^{(t+1)}\bigl(E(\cdot; \Theta^{(t+1)}_{enc})) - f\Bigr\|,
\end{align}
and
\begin{align}
T_2 = 2\,\Bigl\|\nabla E(\cdot;\Theta^{(t)}_{enc})\Bigr\|\,\Bigl\|\nabla D^{(t+1)}\bigl(E(\cdot; \Theta^{(t+1)}_{enc})) - \nabla D^{(t)}\bigl(E(\cdot; \Theta^{(t)}_{enc}))\Bigr\|\,\Bigl\|D^{(t+1)}\bigl(E(\cdot; \Theta^{(t+1)}_{enc}) - f\Bigr\|.
\end{align}
and
\begin{align}
T_3 = 2\|\nabla E(\cdot;\Theta^{(t)}_{enc})\,\nabla D^{(t)}(E(\cdot;\Theta^{(t)}_{enc}))\Bigl\{
\Bigl(D^{(t+1)}\bigl(E(\cdot;\Theta^{(t+1)}_{enc})\bigr)- f\Bigr) - \Bigl(D^{(t)}\bigl(E(\cdot;\Theta^{(t)}_{enc})\bigr)- f\Bigr)
\Bigr\} \Bigr\|.
\end{align}

\noindent \textbf{Bounding $T_1$:}  
By the encoder smoothness assumption,
\begin{align}
\|\nabla E(\cdot;\Theta^{(t+1)}_{enc}) - \nabla E(\cdot;\Theta^{(t)}_{enc})\| \le \beta_{E}\,\|\Theta^{(t+1)}_{enc}-\Theta^{(t)}_{enc}\|,
\end{align}
and by the decoder gradient bound,
\begin{align}
\Bigl\|\nabla D^{(t+1)}\bigl(E(\cdot; \Theta^{(t+1)}_{enc}))\Bigr\| \le B_{D},
\end{align}
and the reconstruction error bound,
\begin{align}
\Bigl\|D^{(t+1)}\bigl(E(\cdot; \Theta^{(t+1)}_{enc})) - f\Bigr\| \le B_{Reconst}.
\end{align}
Thus,
\begin{align}
T_1 \le 2\,\beta_{E}\,B_{D}\,B_{Reconst}\,\|\Theta^{(t+1)}_{enc}-\Theta^{(t)}_{enc}\|.
\end{align}

\noindent \textbf{Bounding $T_2$:}  
We now decompose the term
\begin{align}
\nabla D^{(t+1)}\bigl(E(\cdot; \Theta^{(t+1)}_{enc})) - \nabla D^{(t)}\bigl(E(\cdot; \Theta^{(t)}_{enc})).
\end{align}
By adding and subtracting the term $\nabla D^{(t+1)}((E(\cdot; \Theta^{(t)}_{enc})))$, we obtain:
\begin{align}
\begin{aligned}
\Bigl\|\nabla D^{(t+1)}\bigl(E(\cdot; \Theta^{(t+1)}_{enc})) &- \nabla D^{(t)}\bigl(E(\cdot; \Theta^{(t)}_{enc}))\Bigr\| \\
&\le \Bigl\|\nabla D^{(t+1)}\bigl(E(\cdot; \Theta^{(t+1)}_{enc})) - \nabla D^{(t+1)}((E(\cdot; \Theta^{(t)}_{enc})))\Bigr\| \\
&\quad + \Bigl\|\nabla D^{(t+1)}((E(\cdot; \Theta^{(t)}_{enc}))) - \nabla D^{(t)}\bigl(E(\cdot; \Theta^{(t)}_{enc}))\Bigr\|.
\end{aligned}
\end{align}
By the decoder's Lipschitz continuity with respect to its input, we have:
\begin{align}
\Bigl\|\nabla D^{(t+1)}\bigl(E(\cdot; \Theta^{(t+1)}_{enc})) - \nabla D^{(t+1)}((E(\cdot; \Theta^{(t)}_{enc})))\Bigr\|  \le L_{D}\,\|E(\cdot;\Theta^{(t+1)}_{enc}) - E(\cdot;\Theta^{(t)}_{enc})\|,
\end{align}
and by the encoder Lipschitz condition,
\begin{align}
\|E(\cdot;\Theta^{(t+1)}_{enc}) - E(\cdot;\Theta^{(t)}_{enc})\| \le L_{E}\,\|\Theta^{(t+1)}_{enc}-\Theta^{(t)}_{enc}\|.
\end{align}
Thus, the first term is bounded by:
\begin{align}
L_{D}L_{E}\,\|\Theta^{(t+1)}_{enc}-\Theta^{(t)}_{enc}\|.
\end{align}
For the second term, the decoder gradient smoothness gives:
\begin{align}
\Bigl\|\nabla D^{(t+1)}((E(\cdot; \Theta^{(t)}_{enc}))) - \nabla D^{(t)}((E(\cdot; \Theta^{(t)}_{enc})))\Bigr\| \le \beta_{D}\,\|\Theta^{(t+1)}_{dec}-\Theta^{(t)}_{dec}\|.
\end{align}
Thus,
\begin{align}
\Bigl\|\nabla D^{(t+1)}\bigl(E(\cdot; \Theta^{(t+1)}_{enc})) &- \nabla D^{(t)}\bigl(E(\cdot; \Theta^{(t)}_{enc}))\Bigr\| \le L_{D}L_{E}\,\|\Theta^{(t+1)}_{enc}-\Theta^{(t)}_{enc}\| + \beta_{D}\,\|\Theta^{(t+1)}_{dec}-\Theta^{(t)}_{dec}\|.
\end{align}
Now, using the encoder gradient bound, $\|\nabla E(\cdot;\Theta^{(t)}_{enc})\| \le B_{E}$, and the reconstruction error bound $\|D^{(t+1)}\bigl(E(\cdot; \Theta^{(t+1)}_{enc})) - f\| \le B_{Reconst}$, we have:
\begin{align}
T_2 \le 2\,B_{E}\,B_{Reconst} \left(L_{D}L_{E}\,\|\Theta^{(t+1)}_{enc}-\Theta^{(t)}_{enc}\| + \beta_{D}\,\|\Theta^{(t+1)}_{dec}-\Theta^{(t)}_{dec}\|\right).
\end{align}

\noindent \textbf{Bounding $T_3$:}
\begin{align}
T_3 &= 2\|\nabla E(\cdot;\Theta^{(t)}_{enc})\,\nabla D^{(t)}(E(\cdot;\Theta^{(t)}_{enc}))\Bigl\{
\Bigl(D^{(t+1)}\bigl(E(\cdot;\Theta^{(t+1)}_{enc})\bigr)- f\Bigr) - \Bigl(D^{(t)}\bigl(E(\cdot;\Theta^{(t)}_{enc})\bigr)- f\Bigr)
\Bigr\} \Bigr\| \\
&\leq 2\|\nabla E(\cdot;\Theta^{(t)}_{enc})\| \cdot \|\nabla D^{(t)}(E(\cdot;\Theta^{(t)}_{enc}))\| \cdot \|\Bigl(D^{(t+1)}\bigl(E(\cdot;\Theta^{(t+1)}_{enc})\bigr)- f\Bigr) - \Bigl(D^{(t)}\bigl(E(\cdot;\Theta^{(t)}_{enc})\bigr)- f\Bigr)\| \\
&\leq 2B_E \cdot B_D \cdot 2B_{Reconst} \\
&= 4B_E B_D B_{Reconst}
\end{align}

\textbf{Combining $T_1$, $T_2$ and $T_3$:}
\begin{align}
\begin{aligned}
\Delta_f &\le T_1 + T_2 +  T_3\\
&\le 2\beta_{E}B_{D}B_{Reconst}\,\|\Theta^{(t+1)}_{enc}-\Theta^{(t)}_{enc}\| + 2B_{E}B_{Reconst}L_{D}L_{E}\,\|\Theta^{(t+1)}_{enc}-\Theta^{(t)}_{enc}\| + 2B_{E}B_{Reconst}\beta_{D}\,\|\Theta^{(t+1)}_{dec}-\Theta^{(t)}_{dec}\| + 4B_E B_D B_{Reconst}\\
&= \Bigl[2B_{Reconst}\Bigl(\beta_{E}B_{D} + B_{E}L_{D}L_{E}\Bigr)\Bigr]\,\|\Theta^{(t+1)}_{enc}-\Theta^{(t)}_{enc}\| + 2B_{E}B_{Reconst}\beta_{D}\,\|\Theta^{(t+1)}_{dec}-\Theta^{(t)}_{dec}\| + 4B_E B_D B_{Reconst}.
\end{aligned}
\end{align}
Define
\begin{align}
C_1 = 2B_{Reconst}\Bigl(\beta_{E}B_{D} + B_{E}L_{D}L_{E}\Bigr) \quad \text{and} \quad C_2 = 2B_{E}B_{Reconst}\beta_{D} \quad  \text{and} \quad C_3 =4B_E B_D B_{Reconst}.
\end{align}
Then, the final bound is:
\begin{align}
\Bigl\|\nabla \mathcal{L}_{\mathrm{E-D}}(\Theta_{enc}^{(t+1)})  - \nabla \mathcal{L}_{E-D}(\Theta_{enc}^{(t)})\Bigr\|
\le C_1\,\|\Theta^{(t+1)}_{enc}-\Theta^{(t)}_{enc}\| + C_2\,\|\Theta^{(t+1)}_{dec}-\Theta^{(t)}_{dec}\| + 4B_E B_D B_{Reconst}.
\end{align}
This completes the proof for the encoder gradient difference bound.
\end{proof}

\subsection{Decoder Gradient Difference Upper Bound}
For decoder, assume that:
\begin{enumerate}
    \item \textbf{Decoder Lipschitz Continuity:} 
    \begin{align}
    \|D^{(t+1)}- D^{(t)}\| \le L_{D} \,\|\Theta^{(t+1)}_{dec}-\Theta^{(t)}_{dec}\|. 
    \end{align}
    \item \textbf{Decoder Gradient Smoothness:}
    \begin{align}
    \Bigl\|\nabla D^{(t+1)} - \nabla D^{(t)}\Bigr\| \le \beta_{D}\,\|\Theta^{(t+1)}_{dec}-\Theta^{(t)}_{dec}\|.
    \end{align}
    For simpility, we also assume $\beta_{D}$-smooth with respect to its input which helps to keep the proof concise:
    \begin{align}
    \bigl\|\nabla D(f_1; \Theta_{dec}) \;-\; \nabla D(f_2; \Theta_{dec})\bigr\|
      \;\le\;
      \beta_{D}\,\bigl\|f_1 - f_2\bigr\|.
    \end{align}
    \item \textbf{Boundedness:} There exist constants $B_{D}$ and $B_{Reconst}$ such that
    \begin{align}
    \|\nabla D^{(t+1)}\bigl(E(\cdot; \Theta^{(t+1)}_{enc})) \| \le B_{D},
    \end{align}
    and
    \begin{align}
    \Bigl\|D^{(t+1)}\bigl(E(\cdot; \Theta^{(t+1)}_{enc}))  - f\Bigr\| \le B_{Reconst}.
    \end{align}
    \item \textbf{Encoder Influence:} The encoder is $L_{E}$-Lipschitz with respect to its parameters; that is,
    \begin{align}
    \|E(\cdot;\Theta^{(t+1)}_{enc}) - E(\cdot;\Theta^{(t)}_{enc})\| \le L_{E}\,\|\Theta^{(t+1)}_{enc}-\Theta^{(t)}_{enc}\|.
    \end{align}
\end{enumerate}
Then the gradient difference with respect to the decoder parameters satisfies
\begin{equation}
\label{eq:decoder_bound_final}
\|\nabla \mathcal{L}_{\mathrm{E-D}}(\Theta_{dec}^{(t+1)})  - \nabla \mathcal{L}_{E-D}(\Theta_{dec}^{(t)})\|
\leq 2B_{Reconst}\,\beta_{D}\,L_{E}\,\|\Theta^{(t+1)}_{enc}-\Theta^{(t)}_{enc}\| + 2B_{Reconst}\beta_{D}||\Theta^{(t+1)}_{dec}-\Theta^{(t)}_{dec}|| + 4B_{D}B_{Reconst}
\end{equation}

\begin{proof}
We begin with the gradient with respect to the decoder parameters at iteration $t$, so that
\begin{align}
\nabla \mathcal{L}_{\mathrm{E-D}}(\Theta_{dec}^{(t)})
~=~
2\,\Bigl[D^{(t)}\bigl(E(\cdot;\Theta^{(t)}_{enc})\bigr) - f\Bigr]\,\nabla\,D^{(t)}\bigl(E(\cdot;\Theta^{(t)}_{enc})\bigr).
\end{align}
Similarly, at iteration $t+1$,
\begin{align}
\nabla \mathcal{L}_{\mathrm{E-D}}(\Theta_{dec}^{(t+1)})
~=~
2\,\Bigl[D^{(t+1)}\bigl(E(\cdot;\Theta^{(t+1)}_{enc})\bigr) - f\Bigr]\,\nabla D^{(t+1)}\bigl(E(\cdot;\Theta^{(t+1)}_{enc})\bigr).
\end{align}

Define the difference:
\begin{align}
\Delta_g \triangleq \Bigl\|\nabla \mathcal{L}_{\mathrm{E-D}}(\Theta_{dec}^{(t+1)})
- \nabla \mathcal{L}_{\mathrm{E-D}}(\Theta_{dec}^{(t)})\Bigr\|.
\end{align}
Thus,
\begin{align}
\begin{aligned}
\Delta_g 
&= \Bigl\|2\Bigl[D^{(t+1)}\bigl(E(\cdot;\Theta^{(t+1)}_{enc})\bigr) - f\Bigr]\,\nabla D^{(t+1)}\bigl(E(\cdot;\Theta^{(t+1)}_{enc})\bigr) \\
&\quad\quad\quad\quad -\, 2\Bigl[D^{(t)}\bigl(E(\cdot;\Theta^{(t)}_{enc})\bigr) - f\Bigr]\,\nabla D^{(t)}\bigl(E(\cdot;\Theta^{(t)}_{enc})\bigr)\Bigr\|.
\end{aligned}
\end{align}
To proceed, we add and subtract the intermediate term
\begin{align}
2\Bigl[D^{(t+1)}\bigl(E(\cdot;\Theta^{(t+1)}_{enc})\bigr) - f\Bigr]\,\nabla D^{(t+1)}\bigl(E(\cdot;\Theta^{(t)}_{enc})\bigr)
\end{align}
to obtain:
\begin{align}
\begin{aligned}
\Delta_g = \Bigl\|\, &2\Bigl[D^{(t+1)}\bigl(E(\cdot;\Theta^{(t+1)}_{enc})\bigr) - f\Bigr]\Bigl(\nabla D^{(t+1)}\bigl(E(\cdot;\Theta^{(t+1)}_{enc})\bigr) - \nabla D^{(t+1)}\bigl(E(\cdot;\Theta^{(t)}_{enc})\bigr)\Bigr)\\
&+ 2\Bigl(\bigl[D^{(t+1)}\bigl(E(\cdot;\Theta^{(t+1)}_{enc})\bigr) - f\bigr]\Bigr) \Bigl(\nabla D^{(t+1)}\bigl(E(\cdot;\Theta^{(t)}_{enc})\bigr ) - \nabla D^{(t)}\bigl(E(\cdot;\Theta^{(t)}_{enc})\Bigr\| \\
&+ 2\Bigl(\bigl[D^{(t+1)}\bigl(E(\cdot;\Theta^{(t+1)}_{enc})\bigr) - f\bigr] - \bigl[D^{(t)}\bigl(E(\cdot;\Theta^{(t)}_{enc})\bigr) - f\bigr]\Bigr) \nabla D^{(t)}\bigl(E(\cdot;\Theta^{(t)}_{enc})\bigr) \Bigr\|.
\end{aligned}
\end{align}
Applying the triangle inequality, we have:
\begin{align}
\Delta_g \le T_A + T_B + T_c,
\end{align}
where
\begin{align}
T_A = 2\,\Bigl\|\bigl[D^{(t+1)}\bigl(E(\cdot; \Theta^{(t+1)}_{enc}))  - f\bigr]\Bigl(\nabla D^{(t+1)}\bigl(E(\cdot;\Theta^{(t+1)}_{enc})\bigr) - \nabla D^{(t+1)}\bigl(E(\cdot;\Theta^{(t)}_{enc})\bigr)\Bigr)\Bigr\|,
\end{align}
and
\begin{align}
T_B = 2\,\Bigl\|\Bigl(\bigl[D^{(t+1)}\bigl(E(\cdot;\Theta^{(t+1)}_{enc})\bigr) - f\bigr]\Bigr) \Bigl(\nabla D^{(t+1)}\bigl(E(\cdot;\Theta^{(t)}_{enc})\bigr ) - \nabla D^{(t)}\bigl(E(\cdot;\Theta^{(t)}_{enc})\Bigr\|.
\end{align}
and
\begin{align}
T_C = 2\Bigl(\bigl[D^{(t+1)}\bigl(E(\cdot;\Theta^{(t+1)}_{enc})\bigr) - f\bigr] - \bigl[D^{(t)}\bigl(E(\cdot;\Theta^{(t)}_{enc})\bigr) - f\bigr]\Bigr) \nabla D^{(t)}\bigl(E(\cdot;\Theta^{(t)}_{enc})\bigr) \Bigr\|
\end{align}

\textbf{Bounding $T_A$:}  
Using the decoder gradient bound, we have
\begin{align}
\Bigl\|\nabla D^{(t+1)}\bigl(E(\cdot;\Theta^{(t+1)}_{enc})\bigr) - \nabla D^{(t+1)}\bigl(E(\cdot;\Theta^{(t)}_{enc})\bigr)\Bigr\| \le \beta_{D}\,\Bigl\|E(\cdot;\Theta^{(t+1)}_{enc}) - E(\cdot;\Theta^{(t)}_{enc})\Bigr\|.
\end{align}
By the encoder Lipschitz property,
\begin{align}
\Bigl\|E(\cdot;\Theta^{(t+1)}_{enc}) - E(\cdot;\Theta^{(t)}_{enc})\Bigr\| \le L_{E}\,\|\Theta^{(t+1)}_{enc}-\Theta^{(t)}_{enc}\|.
\end{align}
Also, by the reconstruction error bound,
\begin{align}
\Bigl\|D^{(t+1)}\bigl(E(\cdot; \Theta^{(t+1)}_{enc}))  - f\Bigr\| \le B_{Reconst}.
\end{align}
Therefore,
\begin{align}
T_A \le 2B_{Reconst}\,\beta_{D}\,L_{E}\,\|\Theta^{(t+1)}_{enc}-\Theta^{(t)}_{enc}\|.
\end{align}

\textbf{Bounding $T_B$:}  
For $T_B$, we have
\begin{align}
\Bigl\|\nabla D^{(t+1)}\bigl(E(\cdot;\Theta^{(t)}{enc})\bigr) - \nabla D^{(t)}\bigl(E(\cdot;\Theta^{(t)}{enc})\bigr)\Bigr\| \le \beta_D\Bigl\|\Theta^{(t+1)}_{dec}-\Theta^{(t)}_{dec}\Bigr\|
\end{align}
Since:
\begin{align}
\Bigl\|D^{(t+1)}\bigl(E(\cdot; \Theta^{(t+1)}{enc}))  - f\Bigr\| \le B_{Reconst}
\end{align}

Thus,
it follows that
\begin{align}
T_B \le 2B_{Reconst}\beta_{D}||\Theta^{(t+1)}_{dec}-\Theta^{(t)}_{dec}||
\end{align}

\textbf{Bounding $T_C$:} 
We have:

\begin{align}
\begin{aligned}
&\Bigl\|\bigl[D^{(t+1)}\bigl(E(\cdot; \Theta^{(t+1)}{enc}))  - f\bigr] - \bigl[D^{(t)}(E(\cdot;\Theta^{(t)}{enc})) - f\bigr]\Bigr\|
\le 2B_{Reconst}
\end{aligned}
\end{align}

and

\begin{align}
    \|\nabla D^{(t+1)}\bigl(E(\cdot; \Theta^{(t+1)}_{enc})) \| \le B_{D},
    \end{align}

Thus:
\begin{align}
T_C \le 4B_{D}B_{Reconst}
\end{align}

\textbf{Combining $T_A$, $T_B$ and $T_C$:}  
We then have:
\begin{align}
\begin{aligned}
\Delta_g &\le T_A + T_B + T_C\\
&\le 2B_{Reconst}\,\beta_{D}\,L_{E}\,\|\Theta^{(t+1)}_{enc}-\Theta^{(t)}_{enc}\| + 2B_{Reconst}\beta_{D}||\Theta^{(t+1)}_{dec}-\Theta^{(t)}_{dec}|| + 4B_{D}B_{Reconst}
\end{aligned}
\end{align}
This completes the proof for the decoder-side gradient difference bound.
\end{proof}

\setcounter{section}{2}
\section{Proof of Proxy Task Loss Bounds}
\label{app:proof2}

\begin{theorem}
\textbf{Proxy Task Loss Bounds} under a Lipschitz-dependent assumption between the masked graph signal and the raw graph signal, $\|\bar{f}-f\| \leq \delta$. For our \method, 
    \begin{align}
        \|S(\bar{f};\Phi) - T(f;\Psi)\| \leq L_E \cdot \delta + \epsilon_{T-S}. 
    \end{align}
    For the encoder-decoder models, 
    \begin{equation}
    \|D\bigl(E(\bar{f};\Phi_{enc});\Theta_{dec}\bigr) - f\| \leq L_E\cdot L_D \cdot \delta + \epsilon_{E-D}.
    \end{equation}
    W.L.O.G., assume $\epsilon_{E-D}=\epsilon_{T-S}$, our \method has a smaller error upper bound, indicating that our teacher–student model is closer to the optimal solution $\theta^*$ during training, which in turn implies that its parameter updates are more stable and its convergence speed is faster (as shown in the first result above).
\end{theorem}

\begin{proof}  
\begin{align}
||D\bigl(E(\bar{f};\Phi_{enc});\Theta_{dec}\bigr) - f|| &\leq
||f - D\bigl(E(f;\Phi_{enc});\Theta_{dec}\bigr)||
+||D\bigl(E(f;\Phi_{enc});\Theta_{dec}\bigr) - D\bigl(E(\bar{f};\Phi_{enc});\Theta_{dec}\bigr)|| \\
&\leq \epsilon_{E-D}+ L_{E}L_{D}||f - \bar{f}|| \\
&\leq \epsilon_{E-D} + L_E\cdot L_D \cdot \delta
\end{align}

\begin{align}
\left|\left|S(\bar{f};\Phi) - T(f;\Psi)\right|\right| &\leq \left|\left|S(\bar{f};\Phi)  - S(f;\Phi) \right|\right| + \left|\left|S(f;\Phi)  - T(f;\Psi)\right|\right| \\
&\leq L_{E}\left|\left|\bar{f} - f\right|\right| + \epsilon_{T-S} \\
&\leq L_{E}\delta + \epsilon_{T-S}
\end{align}

\end{proof}  

\setcounter{section}{3}
\section{Proof of Linear Convergence Bounds}
\label{app:proof3}

\subsection{Encoder-Decoder:}
\begin{theorem}
\textbf{Linear Convergence Bounds Under Strong Convexity.}
    For our \method, 
    \begin{align}
        \|\Phi^{(t+1)}-\Phi^*\|^2 &\leq (1-\frac{\mu^2_E}{\beta^2_E})\cdot \|\Phi^{(t)} - \Phi^*\|^2 
    \end{align}
    where $\alpha$ is the momentum parameter (Eqn.~\ref{eq:ema}). 
    For the encoder-decoder models, 
    \begin{align}
        \|\theta^{(t+1)} - \theta^*\|^2 &\leq \left(1 - \frac{\min(\mu^2_E, \mu^2_D)}{\max(\beta^2_E, \beta^2_D)}\right) \|\theta^{(t)} - \theta^*\|^2 
    \end{align}
    from which we can see our \method converges to the optimal solution $\Phi^*$ faster than the encoder-decoder counterpart to their optimal solutions $\Theta^*$ due to a smaller contraction factor $\left(1 - \frac{\mu^2_E}{\beta^2_E}\right)< \left(1 - \frac{\min(\mu^2_E, \mu^2_D)}{\max(\beta^2_E, \beta^2_D)}\right)$. This implies that \method can achieve a faster convergence.
\end{theorem}

\begin{proof}  
From above, We can get the smoothness assumptions:

\begin{align}
    \|\nabla E(\cdot;\Theta^{(t+1)}_{enc}) - \nabla E(\cdot;\Theta^{(t)}_{enc})\| \le \beta_{E}\,\|\Theta^{(t+1)}_{enc}-\Theta^{(t)}_{enc}\|.
\end{align}

and 
\begin{align}
    \Bigl\|\nabla D^{(t+1)} - \nabla D^{(t)}\Bigr\| \le \beta_{D}\,\|\Theta^{(t+1)}_{dec}-\Theta^{(t)}_{dec}\|.
\end{align}

Besides, we also assume strong convexity:
\begin{enumerate}
    \item $\mu_E$-strong convexity of encoder:
    \begin{equation}
        \langle \nabla E(\bar{f};\Theta_{enc}^{(t+1)}) - \nabla E(\bar{f};\Theta_{enc}^{(t)}), \Theta_{enc}^{(t+1)} - \Theta_{enc}^{(t)} \rangle \geq \mu_{E} \cdot \|\Theta_{enc}^{(t+1)} - \Theta_{enc}^{(t)}\|^2
    \end{equation}
    
    \item $\mu_D$-strong convexity of decoder:
    \begin{equation}
        \langle \nabla D(\bar{f};\Theta_{dec}^{(t+1)}) - \nabla D(\bar{f};\Theta_{dec}^{(t)}), \Theta_{dec}^{(t+1)} - \Theta_{dec}^{(t)} \rangle \geq \mu_{D} \cdot \|\Theta_{dec}^{(t+1)} - \Theta_{dec}^{(t)}\|^2
    \end{equation}
\end{enumerate}
When combining an encoder and decoder, the overall strong convexity constant is often at most $\min(\mu_E,\mu_D)$ in a conservative sense.

\noindent Then for the encoder--decoder model, we define $\theta = \bigl(\Theta_{enc}, \Theta_{dec}\bigr)$
for simplicity, where $\theta$ is used as a generic parameter vector for the entire model. The gradient descent update is given by:
\begin{equation}
    \theta_{t+1} = \theta_t - \eta\nabla L_{ED}(\theta_t)
\end{equation}

Following the gradient analysis:
\begin{align}
    \|\theta_{t+1} - \theta^*\|^2 &= \|(\theta_t - \eta\nabla L_{ED}(\theta_t)) - \theta^*\|^2 \\
    &= \|\theta_t - \theta^*\|^2 - 2\eta\langle\nabla L_{ED}(\theta_t), \theta_t - \theta^*\rangle + \eta^2\|\nabla L_{ED}(\theta_t)\|^2
\end{align}

For $\langle\nabla L_{ED}(\theta_t), \theta_t - \theta^*\rangle$:

Since \textbf{$\mu$-strongly convex}, the following inequality holds:
\begin{align}
L(\theta') \geq L(\theta) + \nabla L(\theta)^\top (\theta' - \theta) + \frac{\mu}{2} \|\theta' - \theta\|^2.
\end{align}

Let $ \theta^* $ denote the global optimum of $ L(\theta) $, i.e., 
\begin{align}
\theta^* = \arg\min_{\theta} L(\theta).
\end{align}
then:
\begin{align}
\nabla L(\theta^*) = 0.
\end{align}
Substituting $ \theta' = \theta^* $ into the strong convexity definition, we obtain:
\begin{align}
L(\theta^*) \geq L(\theta_t) + \nabla L(\theta_t)^\top (\theta^* - \theta_t) + \frac{\mu}{2} \|\theta^* - \theta_t\|^2.
\end{align}
\noindent Rearranging the terms, we have:
\begin{align}
L(\theta^*) - L(\theta_t) \geq \nabla L(\theta_t)^\top (\theta^* - \theta_t) + \frac{\mu}{2} \|\theta^* - \theta_t\|^2.
\end{align}
\noindent Since $ \theta^* $ is the global minimum, it follows that $ L(\theta^*) \leq L(\theta_t) $. Therefore:
\begin{align}
L(\theta^*) - L(\theta_t) \leq 0.
\end{align}
\noindent Combining the two inequalities:
\begin{align}
0 \geq \nabla L(\theta_t)^\top (\theta^* - \theta_t) + \frac{\mu}{2} \|\theta^* - \theta_t\|^2.
\end{align}

\begin{align}
\nabla L(\theta_t)^\top (\theta_t - \theta^*) \geq \frac{\mu}{2} \|\theta_t - \theta^*\|^2.
\end{align}

\noindent In general, the encoder and decoder are each $ \mu_E $-strongly convex and $ \mu_D $-strongly convex with respect to their parameters, respectively, then the composition can only guarantee a smaller strong convexity coefficient $\min(\mu_E, \mu_D)$ in the worst case, then:

$$\langle\nabla L_{ED}(\theta_t), \theta_t - \theta^*\rangle \geq \min(\mu_E, \mu_D)\|\theta_t - \theta^*\|^2$$

\noindent Similarly, for $\|\nabla L_{ED}(\theta_t)\|^2$,
since $\nabla L_{ED}(\theta^*) = 0$, then
\begin{align}
\|\nabla L_{ED}(\theta)\|
&= \|\nabla L_{ED}(\theta) - \nabla L_{ED}(\theta^*)\|
\;\le\; \beta\,\|\theta - \theta^*\|.\\
\|\nabla L_{ED}(\theta)\|^2
\;&\le\;
\beta^2\,\|\theta - \theta^*\|^2.
\end{align}
\noindent In Encoder-Decoder, we have two sets of parameters $(\Theta_{enc},\Theta_{dec})$ and we typically argue that
\begin{equation}
L_{ED}(\theta)
\;\text{is at most}\;
(\beta_{E}\text{-smooth}) \times (\beta_{D}\text{-smooth}),
\end{equation}
For simplicity, let $L_{ED}(\theta)$ is $\max(\beta_{E},\beta_{D})$-smooth:
$$\|\nabla L_{ED}(\theta_t)\|^2 \leq \max(\beta_{E}^2, \beta_{D}^2)\|\theta_t - \theta^*\|^2$$

Then we can get:
\begin{equation}
\|\theta_{t+1} - \theta^*\|^2 \leq (1 - 2\eta\min(\mu_E, \mu_D) + \eta^2\max(\beta_{E}^2, \beta_{D}^2))\|\theta_t - \theta^*\|^2
\end{equation}

We want to find the minimum of $(1 - 2\eta\min(\mu_E, \mu_D) + \eta^2\max(\beta_{E}^2, \beta_{D}^2))$:

\begin{equation}
-2\min(\mu_E, \mu_D) + 2\eta\max(\beta_{E}^2, \beta_{D}^2) = 0
\end{equation}

\begin{equation}\eta = \frac{\min(\mu_E, \mu_D)}{\max(\beta_{E}^2, \beta_{D}^2)}\end{equation}

With optimal learning rate $\eta = \frac{\min(\mu_E, \mu_D)}{\max(\beta_{E}, \beta_{D})}$, we obtain:
\begin{equation}
    \|\theta_{t+1} - \theta^*\|^2 \leq (1 - \frac{\min(\mu^2_E, \mu^2_D)}{\max(\beta^2_{E}, \beta^2_{D})})\|\theta_t - \theta^*\|^2 
\end{equation}

\subsection{H$^3$GNN:}
For our method, analyzing one step:
\begin{equation}
    \|\Phi_{t+1} - \Phi^*\|^2 = \|(\Phi_t - \Tilde{\eta} \nabla L_{TS}(\Phi_t)) - \Phi^*\|^2
\end{equation}

\noindent Similarly as above, with optimal learning rate $\Tilde{\eta} = \mu_E/\beta_{E}$:
\begin{equation}
    \|\Phi_{t+1} - \Phi^*\|^2 \leq (1 - \frac{\mu^2_E}{\beta^2_{E}})\|\Phi_t - \Phi^*\|^2 
\end{equation}

\noindent Clearly, our proposed method achieves better convergence because:
\begin{equation}
    \frac{\mu^2_E}{\beta^2_{E}} > \frac{\min(\mu^2_E,\mu^2_D)}{\max(\beta^2_{E},\beta^2_{D})}
\end{equation}

This inequality holds because:
\begin{enumerate}
    \item $\mu^2_E \geq \min(\mu^2_E,\mu^2_D)$
    \item $\beta^2_{E} \leq \max(\beta^2_{E},\beta^2_{D})$
\end{enumerate}

\noindent Obviously, our model yields a faster convergence rate.
\end{proof}  

\setcounter{section}{4}
\section{Datasets Statistics}
\label{app:dataset}

We provide the deatils of datasets used in our experiment here. The homophily ratio, denoted as homo, represents the proportion of edges that connect two nodes within the same class out of all edges in the graph. Consequently, graphs with a strong homophily ratio close to 1, whereas those with a ratio near 0 exhibit strong heterophily.

\begin{table}[H]
\centering
  \caption{Datasets statistics.}
    \label{tab:dataset}
  \begin{tabular}{cccccc}
    \toprule
    \toprule
    Datasets&Node&Edges&Feats&Classes&Homo\\
    \midrule
    Cornell & 183 & 295 & 1,703 & 5 & 0.3\\
    Texas & 183 & 309 & 1,703 & 5 & 0.11\\
    Wisconsin & 251 & 499 & 1,703 & 5 & 0.21\\
    Actor & 7,600 & 29,926 & 932 & 5 & 0.22\\
    \midrule
    Cora & 2708 & 10,556 & 1,433 &7 & 0.81\\
    CiteSeer & 3,327 & 9,104 & 3,703 & 6 & 0.74\\
    PubMed & 19,717 & 88,648 &500 &3 & 0.8\\
  \bottomrule
  \bottomrule
\end{tabular}
\end{table}

\setcounter{section}{5}
\section{Performance}
\label{app:plot}
In this section, we present a radar plot to illustrate the advantages of our proposed H$^3$ GNN compared to major baselines across all datasets as shown in Figure \ref{fig:radar}. This figure clearly demonstrates our model's effectiveness.

\begin{figure*}[t]
    \centering
    \includegraphics[width=0.5\textwidth]{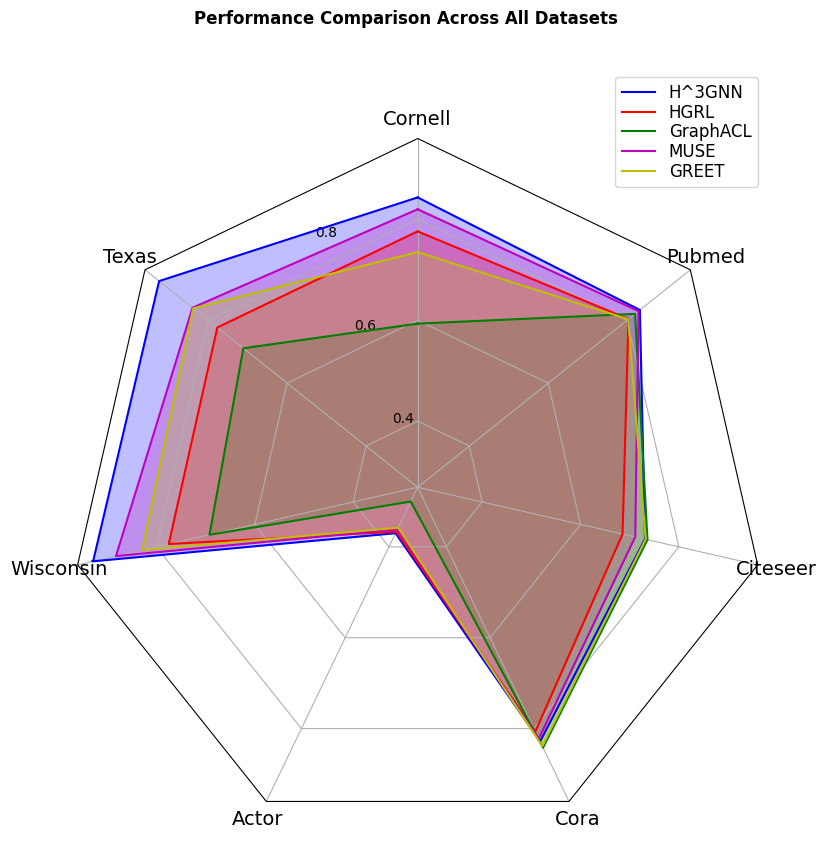}
    \caption{Performance comparison across all datasets}
    \label{fig:radar}
\end{figure*}

\setcounter{section}{6}
\section{Heterophily and Homophily in Graphs}
\label{app:pattern}
\subsection{Datasets Descriptions}
We provide a basic introduction \cite{pei2020geom} and T-SNE  visualizations of four heterophilic datasets which show complicated mix patterns in this section. 

\noindent \textbf{WebKB}. The WebKB1 dataset is a collection of web pages. Cornell, Wisconsin and Texas are three sub-datasets of it. Nodes represent web pages and edges denote hyperlinks between them. The node features are bag-of-words representations of the web pages, which are manually categorized into five classes: student, project, course, staff, and faculty.

\noindent \textbf{Actor Co-occurrence Network}. This dataset is derived from the film-director-actor-writer network. In this network, each node corresponds to an actor, and an edge between two nodes indicates that the actors co-occur on the same Wikipedia page. The node features consist of keywords extracted from these Wikipedia pages, and the actors are classified into five categories based on the content of their pages.

\subsection{Pattern Analysis}
\textbf{Wisconsin, Texas and Cornell}: These three datasets are relatively small and exhibit high heterophily. In the raw feature visualizations (left), nodes of different labels are highly mixed, with significant overlap between categories. After applying \method, the right-side visualizations reveal a more distinct clustering structure, where nodes of the same label are more compactly grouped. For instance, in Texas and Cornell, purple nodes appear more concentrated, and red nodes are better distinguished from other categories, indicating that the model effectively captures the structural patterns. In Wisconsin, the node clusters become more distinguishable, with clearer boundaries between different categories. This demonstrates the model’s ability to learn meaningful representations that enhance classification and clustering tasks.

\noindent \textbf{Actor}: This dataset contains a large number of nodes with an imbalanced label distribution (with red nodes being dominant). In the raw feature space (left), although red nodes are mainly centered, other colored nodes remain scattered without clear boundaries. Notably, the outer ring of nodes effectively represents the mixed structural pattern, which accounts for the relatively low accuracy observed in both node classification and node clustering tasks across all models. In the \method embedding space (right), red nodes are more tightly clustered, while nodes of other labels form relatively well-separated subclusters. This suggests that the model improves class separation and enhances discrimination among different node categories.

Overall, these visualizations demonstrate that in the \method embedding space, nodes of different categories form more distinguishable clusters compared to the raw feature space. This intuitively explains why our model achieves great performance in both node classification and node clustering tasks. Furthermore, it highlights the model’s strong representation learning capability across various graph structures, whether homophilic or heterophilic.

\begin{figure}[h]
    \centering
    \vspace{-0.1in}
    \hspace{-0.3in}
    \subfloat{
        \includegraphics[width=0.5\textwidth]{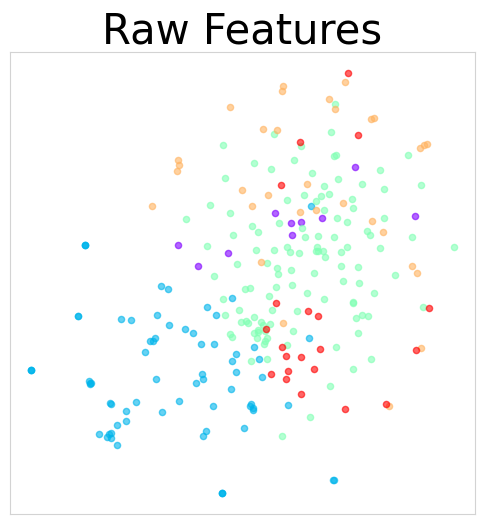}
    }
    \hspace{-0.1in}
    \subfloat{
        \includegraphics[width=0.5\textwidth]{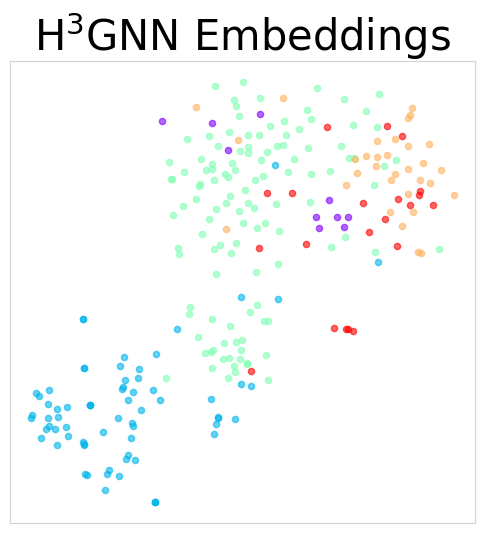}
    }\\
    \captionsetup{font=small}
    \caption{T-SNE visualizations of Wisconsin datasets.}
    \label{fig:pattern}
\end{figure}

\begin{figure}[h]
    \centering
    \vspace{-0.1in}
    \hspace{-0.3in}
    \subfloat{
        \includegraphics[width=0.5\textwidth]{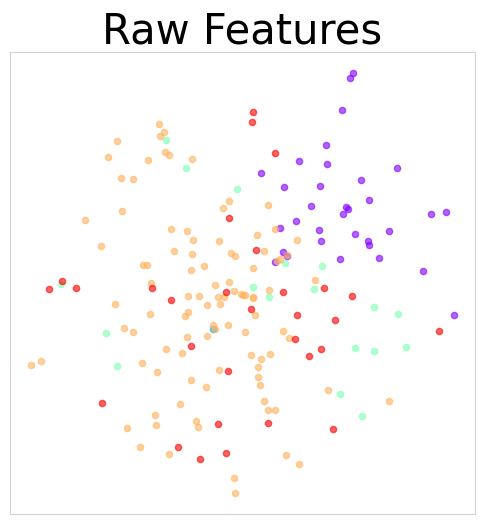}
    }
    \hspace{-0.1in}
    \subfloat{
        \includegraphics[width=0.5\textwidth]{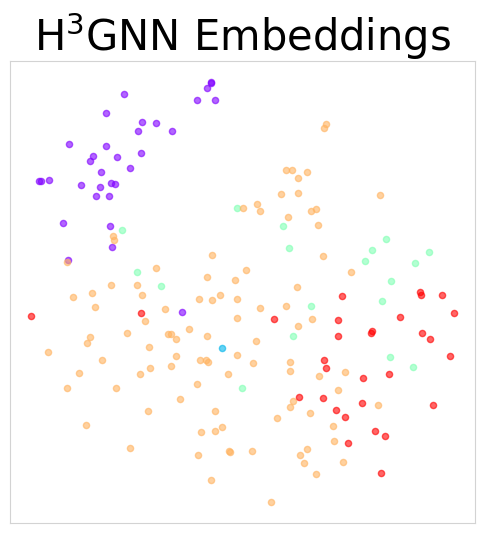}
    }\\
    \captionsetup{font=small}
    \caption{T-SNE visualizations of Texas datasets.}
    \label{fig:pattern}
\end{figure}

\begin{figure}[h]
    \centering
    \vspace{-0.1in}
    \hspace{-0.3in}
    \subfloat{
        \includegraphics[width=0.5\textwidth]{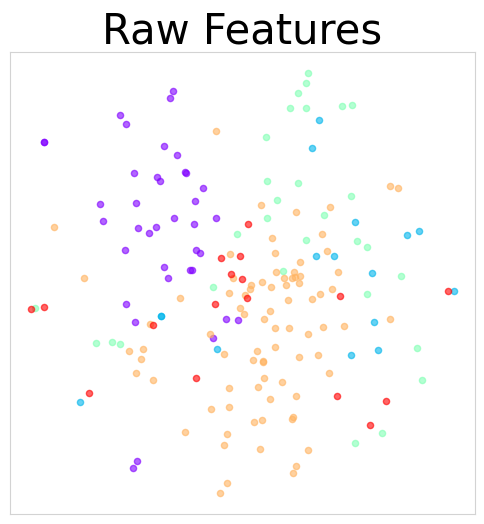}
    }
    \hspace{-0.1in}
    \subfloat{
        \includegraphics[width=0.5\textwidth]{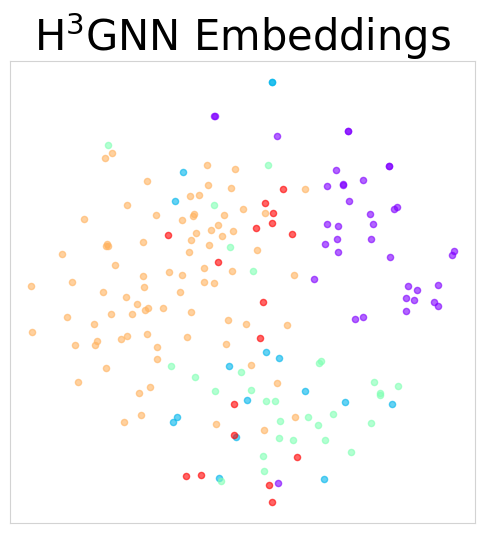}
    }\\
    \captionsetup{font=small}
    \caption{T-SNE visualizations of Cornell datasets.}
\end{figure}

\begin{figure}[h]
    \centering
    \vspace{-0.1in}
    \hspace{-0.3in}
    \subfloat{
        \includegraphics[width=0.5\textwidth]{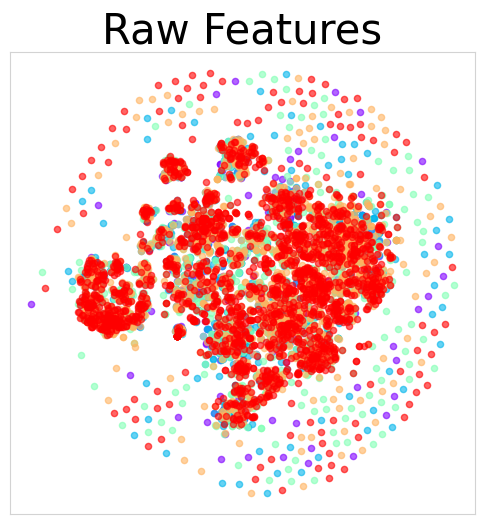}
    }
    \hspace{-0.1in}
    \subfloat{
        \includegraphics[width=0.5\textwidth]{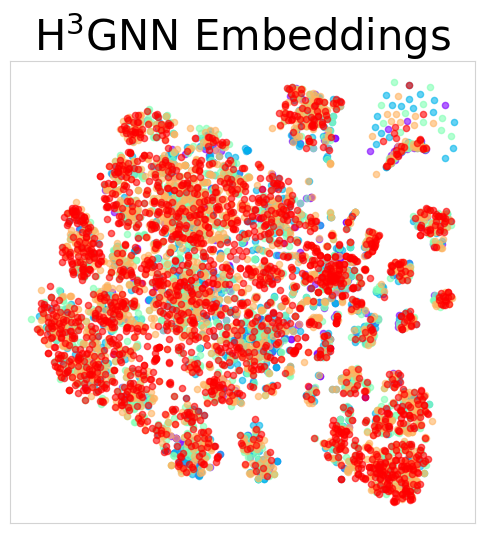}
    }\\
    \captionsetup{font=small}
    \caption{T-SNE visualizations of Actor datasets.}
\end{figure}

\setcounter{section}{7}
\section{Hyperparameters}
\label{app:hyper}

Our model's hyperparameters are tuned from the following search space:
\begin{itemize}
    \item Learning rate for SSL model: $\{0.01,\ 0.005,\ 0.001\}$.
    \item Learning rate for classifier: $\{0.01,\ 0.005,\ 0.001\}$.
    \item Weight decay for SSL model: $\{0,\ 1\times10^{-3}, \ 5\times10^{-3}, \ 8\times10^{-3}, \ 1\times10^{-4}, \ 5\times10^{-4}, \ 8\times10^{-4}\}$.
    \item Weight decay for classifier: $\{0,\ 5\times10^{-4},\ 5\times10^{-5}\}$.
    \item Dropout for Filters: $\{0.1,\ 0.3,\ 0.5,\ 0.7,\ 0.8\}$.
    \item Dropout for Attention: $\{0.1,\ 0.3,\ 0.5,\ 0.7,\ 0.8\}$.
     \item Dimension of tokens: $\{128,\ 256,\ 512,\,1024, \,2048\}$.
    \item Hidden units of filters: $\{16,\ 32,\ 64,\,128\}$.
    \item Total masking ratio: $\{0.9,\ 0.8,\ 0.5,\,0.3,\,0.2,\,0.1\}$.
    \item Dynamic masking ratio: $\{0.9,\ 0.8,\ 0.5,\,0.3,\,0.2,\,0.1\}$.
    \item Momentum: $\{0.9,\ 0.99,\ 0.999\}$.
\end{itemize}

\setcounter{section}{8}
\section{Performance Comparison for Node Clustering}
\label{app:cluster}

In this section, we present a performance comparison for node clustering. We compare our model with four groups of baseline methods:

\begin{itemize}
    \item \textit{Traditional Unsupervised Clustering Methods:} AE \cite{hinton2006reducing}, node2vec \cite{grover2016node2vec}, struc2vec \cite{ribeiro2017struc2vec}, and LINE \cite{tang2015line}.
    \item \textit{Attributed Graph Clustering Methods:} GAE (VGAE) \cite{kipf2016variational}, GraphSAGE \cite{hamilton2017inductive}, and SDCN \cite{bo2020structural}.
    \item \textit{Self Supervised Methods for Homophilic Graphs:} MVGRL \cite{hassani2020contrastive}, GRACE \cite{zhu2020deep}, and BGRL \cite{thakoor2021large}.
    \item \textit{Self Supervised Methods for Heterophilic Graphs:} DSSL \cite{xiao2022decoupled}, HGRL \cite{chen2022towards}, and MUSE \cite{yuan2023muse}.
\end{itemize}

Following the same protocol as with other baselines, we freeze the model and use the generated embeddings for \emph{k}-means clustering. We reproduce MUSE~\cite{yuan2023muse}, as it has been proven to be the state-of-the-art model for node clustering. However, the original paper does not provide any hyperparameters for node clustering on any dataset, we perform hyperparameter tuning ourselves.
For the other baselines, we report the results from baseline papers \cite{chen2022towards,yuan2023muse}. The hyperparameters search space can be found in Appendix \ref{app:hyper}. The results are shown in Table \ref{tab:cluster}.

\begin{table*}\footnotesize\centering
\caption{Clustering results (ACC in percent $\pm$ standard deviation). The best and runner-up results are highlighted with \textcolor{red}{red} and \textcolor{blue}{blue}, respectively.}
\label{tab:cluster}
\setlength{\tabcolsep}{6pt}
\scalebox{1.8}{
\begin{threeparttable}
\begin{tabular}{lcccc}
\toprule
\toprule
\multirow{2}{*}{Methods} & \multicolumn{1}{c}{Texas} & \multicolumn{1}{c}{Actor} & \multicolumn{1}{c}{Cornell} & \multicolumn{1}{c}{CiteSeer}\\
\cline{2-5}
& ACC & ACC & ACC & ACC \\
\midrule
AE \cite{hinton2006reducing} & 50.49$\pm$0.01 & 24.19$\pm$0.11 & 52.19$\pm$0.01 & 58.79$\pm$0.19\\
node2vec \cite{grover2016node2vec} & 48.80$\pm$1.93 & 25.02$\pm$0.04 & 50.98$\pm$0.01 & 20.76$\pm$0.27\\
struc2vec \cite{ribeiro2017struc2vec} & 49.73$\pm$0.01 & 22.49$\pm$0.34 & 32.68$\pm$0.01 & 21.22$\pm$0.45\\
LINE \cite{tang2015line} & 49.40$\pm$2.08 & 22.70$\pm$0.08 & 34.10$\pm$0.77 & 28.42$\pm$0.88\\
\midrule
GAE \cite{kipf2016variational} & 42.02$\pm$1.22 & 23.45$\pm$0.04 & 43.72$\pm$1.25 & 48.37$\pm$0.37\\
VGAE \cite{kipf2016variational} & 50.27$\pm$1.87 & 23.30$\pm$0.22 & 43.39$\pm$0.99 & 55.67$\pm$0.13\\
GraphSAGE \cite{hamilton2017inductive} & 56.83$\pm$0.56 & 23.08$\pm$0.29 & 44.70$\pm$2.00 & 49.28$\pm$1.18\\
SDCN \cite{bo2020structural} & 44.04$\pm$0.56 & 23.67$\pm$0.29 & 36.94$\pm$2.00 & 59.86$\pm$1.18\\
\midrule
MVGRL \cite{hassani2020contrastive} & 62.79$\pm$2.33 & 28.58$\pm$1.03 & 43.77$\pm$3.03 & 45.67$\pm$9.08\\
GRACE \cite{zhu2020deep} & 56.99$\pm$2.23 & 25.87$\pm$0.45 & 43.55$\pm$4.60 & 54.66$\pm$5.41\\
BGRL \cite{thakoor2021large} & 58.68$\pm$1.80 & 28.20$\pm$0.27 & 55.08$\pm$1.68 & 64.27$\pm$1.68\\
\midrule
DSSL \cite{xiao2022decoupled} & 57.43$\pm$3.51 & 26.15$\pm$0.46 & 44.70$\pm$2.44 & 54.32$\pm$3.69\\
HGRL \cite{chen2022towards} & 61.97$\pm$3.10 & 29.79$\pm$1.11 & 60.56$\pm$3.72 & 61.14$\pm$1.49\\
MUSE$^\dagger$ \cite{yuan2023muse} & 65.79$\pm$4.36 & 31.05$\pm$0.72 & 62.35$\pm$2.38 & \textcolor{blue}{66.03}$\pm$2.33\\
\midrule
\method+Diffi & \textcolor{blue}{76.50}$\pm$1.50 & \textcolor{blue}{31.22}$\pm$0.76 & \textcolor{blue}{73.22}$\pm$3.45 & 65.80$\pm$2.32 \\
\method+Prob & \textcolor{red}{77.05}$\pm$2.66 & \textcolor{red}{32.10}$\pm$1.51 & \textcolor{red}{74.86}$\pm$2.09 & \textcolor{red}{66.56}$\pm$3.56 \\
\bottomrule
\bottomrule
\end{tabular}
       \begin{tablenotes}
           \item[] $^\dagger$ MUSE doesn't provide any hyperparameters for node clustering.
        \end{tablenotes}
\end{threeparttable}}
\end{table*}

From the results, we can achieve the similar conclusions as node classification:
\begin{itemize}
    \item Our \method achieves significantly better performance than all baselines, including the state-of-the-art model MUSE, by a large margin on the Texas and Cornell datasets, with improvements of 11.26\% and 12.51\%, respectively. Moreover, \method slightly outperforms MUSE on Actor due to the complex mixed structural patterns, as introduced in Appendix~\ref{app:pattern}. It also attains comparable performance on Citeseer. These findings are consistent with those observed in node classification tasks. Overall, our results demonstrate that \method can generate high-quality embeddings regardless of the downstream tasks and effectively handle both heterophilic and homophilic patterns, highlighting its strong generalization capability in graph representation learning.

    \item Regarding the two masking strategies, probabilistic masking consistently outperforms difficulty masking. This finding aligns with our observations in node classification and can be attributed to a better balance between exploration and exploitation.

\end{itemize}

\setcounter{section}{9}
\section{Performance Comparison for More Datasets}
\label{app:more_data}
In this section, we present a performance comparison of node classification on additional heterophilic datasets. Specifically, we use the filtered versions of two commonly used datasets, Chameleon and Squirrel, as the original versions are known to be problematic \cite{platonov2023critical}. We perform self-tuning on all self-supervised learning models, as their original papers do not provide results on the filtered versions of the datasets. For all other baselines, we use the results reported in \cite{platonov2023critical}.

Regarding the results in Table \ref{tab:more_data}, our findings are consistent with the main outcomes reported in Table \ref{tab:major} and Appendix \ref{app:cluster}. Our \method outperforms both the supervised and self-supervised learning baselines. In particular, it significantly exceeds the state-of-the-art baselines MUSE and GREET, demonstrating its robust ability to handle both heterophilic and homophilic patterns. Moreover, when comparing the two masking strategies, probability masking consistently outperforms difficulty masking. This suggests that employing a base masking probability for all nodes, rather than focusing solely on difficult nodes during training, effectively balances exploration and exploitation.

\begin{table*}[t]
\centering
  \caption{Results of node classification (in percent $\pm$ standard deviation across ten splits). The \textcolor{red}{best} and the \textcolor{blue}{runner-up}  results are highlighted in red and blue respectively in terms of the mean accuracy.} 
  \label{tab:more_data}
  \setlength{\tabcolsep}{1.2mm}
  \resizebox{0.8\linewidth}{!}{%
  \begin{tabular}{c|lcc}
    \toprule
    \toprule
    &\multirow{1}{*}{Methods} & Chameleon(filtered) & Squirrel(filtered) \\
    \midrule
    \multirow{5}{*}{SL} &
    GCN~\cite{kipf2016semi}    & 40.89$\pm$4.12 & 39.47$\pm$1.47\\
    & GAT \cite{velivckovic2017graph}   &39.21.$\pm$3.08 & 35.62$\pm$2.06\\
    & SAGE \cite{hamilton2017inductive} &37.77.$\pm$4.14 & 36.09$\pm$1.99   \\
    & GPR-GNN \cite{chien2020adaptive} &39.93$\pm$3.30 & 38.95$\pm$1.99  \\
    & FAGCN \cite{bo2021beyond} &41.90$\pm$2.72 & 41.08$\pm$2.27  \\
    & H2GCN \cite{zhu2020beyond} &26.75.$\pm$3.64 & 35.10$\pm$1.15  \\
    \midrule
    \multirow{2}{*}{SSL}
    & MUSE \cite{yuan2023muse}  & 46.48$\pm$2.51 & 41.57$\pm$1.44 \\
    & GREET \cite{liu2022beyond}   & 44.67$\pm$2.98 & 39.69$\pm$1.85\\
    \midrule
    \multirow{2}{*}{SSL-Ours} 
    & \method +Diffi   & \textcolor{blue}{47.50}$\pm$3.27 & \textcolor{blue}{44.68} $\pm$1.68\\
    & \method +Prob   & \textcolor{red}{48.91}$\pm$3.86 & \textcolor{red}{45.49}$\pm$2.13   \\
    \bottomrule
    \bottomrule
  \end{tabular}
  }
\end{table*}

\setcounter{section}{10}
\section{Performance Comparison with Recent Baselines}
\label{app:rent_sota}
In this section, we present a performance comparison of node classification using recent and strong state-of-the-art baselines, namely PCNet \cite{li2024pc}, AEROGNN \cite{lee2023towards}, and G$^2$ \cite{rusch2022gradient}. We report the results as provided in their respective original papers. As shown in Table \ref{tab:recent_sota}, our \method achieves the best performance on the Wisconsin and Texas datasets, while delivering comparable results to state-of-the-art baselines on the Actor and Cornell datasets. This indicates \method's ability to handle complex mixed patterns in graphs.

\begin{table*}\footnotesize\centering
\caption{Results of node classification (in percent $\pm$ standard deviation across ten splits). The \textcolor{red}{best} and the \textcolor{blue}{runner-up}  results are highlighted in red and blue respectively in terms of the mean accuracy.}
\label{tab:recent_sota}
\setlength{\tabcolsep}{6pt}{
\scalebox{1.8}{
\begin{threeparttable}
\begin{tabular}{lcccc}
\toprule
\toprule
\multirow{2}{*}{Methods} & \multicolumn{1}{c}{Cornell} & \multicolumn{1}{c}{Texas} & \multicolumn{1}{c}{Wisconsin} & \multicolumn{1}{c}{Actor} \\
\cline{2-5}
& ACC & ACC & ACC & ACC \\
\midrule
PCNet \cite{li2024pc} & 82.16$\pm$ 2.70 & 88.11$\pm$2.17 & 88.63$\pm$ 2.75 & \textcolor{red}{37.80}$\pm$ 0.64\\
AEROGNN \cite{lee2023towards} & 81.24$\pm$6.80 & 84.35$\pm$5.20 & 84.80$\pm$3.30 & 36.57$\pm$1.10\\
G$^2$ \cite{rusch2022gradient} & \textcolor{red}{86.22}$\pm$4.90 & 87.57$\pm$3.86 & 87.84$\pm$3.49 & ---\\
\midrule
\method +Diffi (Ours)   & 84.60$\pm$2.43 & \textcolor{red}{92.46}$\pm$2.91 & \textcolor{blue}{92.55}$\pm$3.49 & 36.75$\pm$1.10  \\
\method +Prob (Ours)  & \textcolor{blue}{84.86}$\pm$2.48 & \textcolor{blue}{92.45}$\pm$3.78 & \textcolor{red}{92.89}$\pm$3.09 & \textcolor{blue}{37.00}$\pm$0.91  \\
\bottomrule
\bottomrule
\end{tabular}
\end{threeparttable}}}
\end{table*}

\end{document}